\documentclass[runningheads]{llncs}

\usepackage{configuration/eccv}

\usepackage{graphicx}
\usepackage{booktabs}

\usepackage[accsupp]{axessibility}  

\usepackage{hyperref}

\usepackage{orcidlink}

\usepackage{makecell}
\newcommand{\IPC}{\texttt{IPC}\xspace}
\newcommand{\algopt}{\textsc{CIM}\xspace}

\newcommand{\Squeeze}{\textsc{Squeeze}\xspace}
\newcommand{\Recover}{\textsc{Recover}\xspace}
\newcommand{\Relabel}{\textsc{Relabel}\xspace}
\definecolor{red}{RGB}{196,78,82}
\definecolor{orange}{RGB}{242, 142, 43}
\definecolor{green}{RGB}{118, 183, 178}
\definecolor{blue}{RGB}{76, 114, 176}

\definecolor{orange}{RGB}{178,92,35}
\definecolor{green1}{RGB}{95,145,51}

\definecolor{red1}{RGB}{197,64,57}
\definecolor{blue1}{RGB}{59,130,220}
\definecolor{green2}{RGB}{82,181,150}
\definecolor{purple1}{RGB}{105,93,223}
\definecolor{orange1}{RGB}{164, 47, 54}
\definecolor{green3}{RGB}{94,145,51}

\definecolor{c_step1}{RGB}{60, 105, 199}
\definecolor{c_step2}{RGB}{159,37,28}
\definecolor{c_step3}{RGB}{107,40,157}

\providecommand{\abs}[1]{\left\lvert#1\right\rvert}

\providecommand{\E}{{\mathbb E}}
\providecommand{\E}[1]{{\mathbb E}\left.#1\right. }        %

\DeclareMathOperator*{\argmin}{arg\,min\,}

\providecommand{\xx}{\mathbf{x}}

\providecommand{\mtheta}{\boldsymbol{\theta}}

\providecommand{\cA}{\mathcal{A}}

\providecommand{\cD}{\mathcal{D}}

\providecommand{\cG}{\mathcal{G}}

\providecommand{\cL}{\mathcal{L}}

\providecommand{\cN}{\mathcal{N}}

\providecommand{\cP}{\mathcal{P}}

\providecommand{\cR}{\mathcal{R}}
\providecommand{\cS}{\mathcal{S}}
\providecommand{\cT}{\mathcal{T}}

\newenvironment{talign*}
{\csname align*\endcsname}
{\endalign}

\usepackage[utf8]{inputenc}         %
\usepackage[T1]{fontenc}            %
\usepackage{url}                    %
\usepackage{booktabs}               %
\usepackage{amsfonts}               %
\usepackage{nicefrac}               %
\usepackage{microtype}              %
\usepackage{xcolor}                 %
\usepackage{algorithm}
\usepackage{algorithmic}
\usepackage{graphicx}
\usepackage{subcaption}
\usepackage[flushleft]{threeparttable}
\usepackage{float}
\usepackage{multirow}
\usepackage{xspace}
\usepackage{enumitem}
\usepackage[font=small]{caption}
\usepackage{autobreak}
\usepackage{sidecap}
\usepackage{wrapfig}
\usepackage{bbding}
\usepackage[toc, page, header]{appendix}
\usepackage{tikz}
\usepackage{xcolor}
\usepackage{pifont}
\usepackage{mdframed}
\usepackage{colortbl}

\definecolor{coral}{RGB}{255,127,80}
\definecolor{darkgreen}{RGB}{0,100,0}
\definecolor{darkyellow}{RGB}{204,153,0}
\definecolor{salmon}{RGB}{250,128,114}
\definecolor{darkred}{RGB}{150,0,0}
\newcommand{\darkredtext}[1]{{\color{darkred}#1}}

\newcommand{\secref}[1]{\hyperref[#1]{\darkredtext{Sec.~\ref*{#1}}}}
\newcommand{\thmref}[1]{\hyperref[#1]{\darkredtext{Thm.~\ref*{#1}}}}
\newcommand{\defref}[1]{\hyperref[#1]{\darkredtext{Def.~\ref*{#1}}}}
\newcommand{\propref}[1]{\hyperref[#1]{\darkredtext{Prop.~\ref*{#1}}}}
\newcommand{\assumpref}[1]{\hyperref[#1]{\darkredtext{Assump.~\ref*{#1}}}}
\newcommand{\remarkref}[1]{\hyperref[#1]{\darkredtext{Rem.~\ref*{#1}}}}
\newcommand{\hypref}[1]{\hyperref[#1]{\darkredtext{Hyp.~\ref*{#1}}}}
\newcommand{\conjref}[1]{\hyperref[#1]{\darkredtext{Conj.~\ref*{#1}}}}
\newcommand{\lemref}[1]{\hyperref[#1]{\darkredtext{Lem.~\ref*{#1}}}}
\newcommand{\corref}[1]{\hyperref[#1]{\darkredtext{Cor.~\ref*{#1}}}}
\newcommand{\noteref}[1]{\hyperref[#1]{\darkredtext{Nota.~\ref*{#1}}}}
\newcommand{\claimref}[1]{\hyperref[#1]{\darkredtext{Clm.~\ref*{#1}}}}
\newcommand{\obsref}[1]{\hyperref[#1]{\darkredtext{Obs.~\ref*{#1}}}}
\newcommand{\algref}[1]{\hyperref[#1]{\darkredtext{Alg.~\ref*{#1}}}}
\newcommand{\figref}[1]{\hyperref[#1]{\darkredtext{Fig.~\ref*{#1}}}}
\newcommand{\tabref}[1]{\hyperref[#1]{\darkredtext{Tab.~\ref*{#1}}}}
\newcommand{\appref}[1]{\hyperref[#1]{\darkredtext{App.~\ref*{#1}}}}
\renewcommand{\eqref}[1]{\hyperref[#1]{\darkredtext{Eq.~\ref*{#1}}}}

\usepackage{xcolor}
 
\definecolor{green}{RGB}{44, 160, 44}
\usepackage{graphicx}

\begin{document}

\title{Condensing Large-Scale Datasets Directly with Minimal Information Loss}


\author{Xinyi Shang\inst{1*}\and
Peng Sun\inst{2,3*}\and
Bei Shi\inst{4,*}\and
Zixuan Wang\inst{2}\and
Tao Lin\inst{3,\dag}
}

\institute{University College London, United Kingdom \and
Zhejiang University, China  \and
Westlake University, China  \and
University of Macau, China
}

\authorrunning{X.~Shang et al.}


\maketitle

\begingroup
\renewcommand{\thefootnote}{*}
\footnotetext[1]{Equal contribution.}
\renewcommand{\thefootnote}{\dag}
\footnotetext[2]{Corresponding author. Email: lintao@westlake.edu.cn}
\endgroup

\setcounter{footnote}{0}
\renewcommand{\thefootnote}{\arabic{footnote}}

\begin{abstract}
    Recent advancements in scaling dataset distillation rely heavily on decoupled information extraction pipelines, comprising \Squeeze, \Recover, and \Relabel stages. Despite their scalability to large-scale datasets, these methods suffer from prohibitive computational overhead and poor cross-architecture generalization.
    In this paper, we reveal the root cause of these bottlenecks: the implicit dual-compression process, from data to model and back to images, inherently induces severe information loss.
    Crucially, we empirically and theoretically demonstrate that this loss creates a distribution shift that fundamentally compromises the widely adopted \Relabel strategy, transforming the pre-trained model into an unreliable labeler that yields sub-optimal labels.
    To overcome these critical flaws, we propose \algopt, a novel, metric-driven framework that abandons the flawed dual-compression paradigm.
    Instead, \algopt explicitly quantifies and minimizes the information gap between the original and synthetic datasets. By directly aligning the data distributions, our approach ensures high-fidelity information condensation and inherently satisfies the prerequisites for effective relabeling.
    Extensive experiments demonstrate that \algopt establishes a new state-of-the-art. Notably, it distills ImageNet-1K at an \IPC=10 in merely 80 minutes on a single RTX-4090 GPU, achieving an unprecedented 48.7\% Top-1 accuracy on ResNet-18 and significantly outperforming previous SOTA approaches, such as NRR-DD and DELT, by 2.6\% and 2.9\%, respectively. Our code is available at \url{https://github.com/LINs-lab/CIM}. 
\end{abstract}

\section{Introduction}
\label{sec:intro}

Dataset distillation (DD) \cite{wang2018dataset} aims to condense the knowledge from a massive training dataset into a remarkably small synthetic set, preserving comparable generalization performance.
By substituting the original data with this compact proxy, dataset distillation drastically reduces training overhead and storage costs.
Consequently, it has emerged as an enabling technique for diverse applications, including continual learning~\cite{zhao2020dataset,rosasco2021distilled}, neural architecture search~\cite{wang2021rethinking,such2020generative}, and privacy-preserving learning~\cite{xiong2023feddm,dong2022privacy,shang2022federated,chen2022private}.

To scale dataset distillation to large-scale datasets like ImageNet-1K~\cite{deng2009imagenet} and ImageNet-21K~\cite{ridnik2021imagenet}, a recent line of research focuses on information extraction pipelines~\cite{sun2024diversity,yin2023squeeze,shao2023generalized,shen2024delt,tran2025enhancing}.
Most notably, the pioneering SRe$^2$L~\cite{yin2023squeeze} introduces a decoupled three-stage paradigm: (i) \Squeeze the dataset information into a pre-trained model, (ii) \Recover this information back into the image space to form synthetic data, and (iii) \Relabel the distilled images using the pre-trained model to inject label-space knowledge.
By avoiding costly unrolled optimization, this paradigm yields strong empirical results.

Despite their success, existing extraction-based pipelines suffer from two critical bottlenecks: poor cross-architecture generalization and prohibitive computational overhead, particularly during the \Recover stage.
We identify the root cause as the severe information loss induced by the implicit \textit{dual-compression} process: first from data to model parameters (\Squeeze), and then from model back to synthetic images (\Recover).
Furthermore, while the \Relabel stage improves performance, our theoretical and empirical analyses reveal a hidden vulnerability: its efficacy is strictly conditional on distribution alignment.
When the dual-compression process causes the synthetic samples to drift away from the original data distribution, the pre-trained model acts as an unreliable labeler, yielding sub-optimal labels.
Therefore, ensuring distributional proximity is crucial to fully unlock the benefits of the \Relabel strategy.

Motivated by these insights, we abandon the flawed dual-compression paradigm and propose \algopt, a novel framework that explicitly \underline{c}ondenses \underline{i}nformation by \underline{m}inimizing the \emph{information loss} between real and distilled datasets (\algopt).
Unlike prior extraction-based methods, which rely on complex inversion to recover informative synthetic images from a pre-trained model, \algopt is explicitly \emph{metric-driven}.
Specifically, we formulate a principled metric to comprehensively quantify the information gap, and consequently the distribution shift, between the synthetic samples and their real counterparts.
By directly optimizing the distilled set to minimize this gap within a unified objective, \algopt ensures high-fidelity information retention and strict distribution alignment.
This design not only eliminates the computationally expensive recovery process but also natively satisfies the conditions required for effective relabeling. \textbf{Our contributions are summarized as follows:}
\begin{enumerate}[leftmargin=12pt,itemsep=5pt]
    \item We conduct a rigorous revisiting of information extraction-based dataset distillation, revealing that the inherent dual-compression process leads to severe information loss. This loss not only degrades cross-architecture generalization and computational efficiency but also causes a distribution shift that fundamentally compromises the widely adopted \Relabel strategy.
    \item We introduce \algopt, a novel, metric-driven framework that explicitly quantifies and minimizes the information gap between the original and distilled datasets, achieving more faithful, efficient, and aligned information condensation.
    \item Extensive experiments verify that \algopt establishes a new state-of-the-art across various scale datasets. Notably, our framework distills ImageNet-1K at \IPC$ = 10$ in merely 80 minutes on a single RTX-4090 GPU, achieving an unprecedented 48.7\% Top-1 accuracy on ResNet-18, outperforming previous SOTA approaches, such as NRR-DD and DELT, by 2.6\% and 2.9\%, respectively.
\end{enumerate}

\section{Related Work}
\label{sec:related_work}
The primary aim of dataset distillation is to condense a large dataset into a smaller, yet highly representative subset, preserving its core semantic and statistical characteristics.
Wang et al.~\cite{wang2018dataset} first introduce the dataset distillation as a bi-level meta-learning optimization problem.
The outer loop aims at optimizing the meta-dataset, while the inner loop focuses on training models using the distilled dataset.
Existing methods can be roughly divided into three paradigms for solving this complex bi-level optimization problem.

\vspace{5pt}
\noindent{\bf Uni-level optimization-based paradigm.}
Tackling this bi-level problem is complex, especially when optimizing proxy models via gradient descent, which involves unraveling an intricate computational graph.
studies~\cite{zhou2022dataset,loo2022efficient} have proposed approximating model training using kernel ridge regression, which provides a closed-form solution for optimal weights, thereby reducing training costs and improving performance.
Despite these advancements, such methods still struggle with extensive computational demands or limitations due to the approximations in convex relaxation.

\vspace{5pt}
\noindent{\bf Matching-based paradigm.}
Another strategy involves emulating the behaviors of the original dataset in the distilled one. They focus on minimizing disparities between surrogate models trained on both synthetic and original datasets.
The key metrics for this are matching gradients~\cite{zhao2020dataset,kim2022dataset,zhang2023accelerating,liu2023dream}, features~\cite{wang2022cafe}, distribution~\cite{zhao2023dataset,zhao2023improved}, and training trajectories~\cite{cazenavette2022dataset,cui2022dc,du2023minimizing,cui2023scaling,yu2023dataset,guo2023towards}.
Trajectory and gradient matching, in particular, has shown impressive results with low \IPC.
However, these methods often tailor the distilled dataset to specific network architectures, limiting their generalizability.
Cazenavette et al.~\cite{cazenavette2023generalizing} address this by proposing the GLaD that synthesizes more realistic images to enhance generalization.
These methods incur substantial computational overhead due to the frequent calculation of discrepancies between the distilled and original datasets, requiring numerous iterations for optimization until convergence.
Therefore, computational and memory challenges remain, particularly when scaling to large datasets.

\vspace{5pt}
\noindent{\bf Information extraction-based paradigm.}
The SRe$^2$L framework~\cite{yin2023squeeze}, as the first work efficiently scalable to ImageNet-1K, introduces a novel decoupled bi-level learning paradigm.
This involves three stages:
1) \Squeeze relevant information from the original dataset into a pre-trained model,
2) \Recover this information into the image space,
3) \Relabel the distilled images by using the pre-trained model to further distill knowledge into the label space.
Its efficiency and effectiveness have garnered community attention, spurring a series of research efforts.
Shao et al.~\cite{shao2023generalized} note that SRe$^2$L is limited to specific backbones and layers, impacting the generalization of the distilled dataset.
They advocate for using diverse backbones for more precise and effective distillation.
Yin et al.~\cite{yin2023dataset} further enhance SRe$^2$L with curriculum data augmentation.
\cite{liu2023dataset} use Wasserstein distance to create more representative images.
Sun et al.~\cite{sun2024diversity} introduce an optimization-free approach RDED that achieves notable diversity and realism in distilled datasets.
Recently, DELT~\cite{shen2024delt} mitigates the issue of low within-class diversity of SRe$^2$L by varying the number of iterations for different IPCs during the data synthesis phase.
Moreover, NRR-DD~\cite{tran2025enhancing} addresses the shortcoming of SRe$^2$L in often emphasizing instance-specific features by effectively capturing both class-general and instance-specific features.
EDC~\cite{shao2024elucidating} introduces a unified framework grounded in both empirical and theoretical foundations. Building on the distillation process introduced by SRe$^2$L~\cite{yin2023squeeze}, it incorporates two key enhancements: soft category-aware matching and dynamic adjustments to the learning rate schedule.
These additions can further optimize the distillation process.

\section{On the Pitfalls of Information Extraction Paradigm}
In this section, we revisit the information extraction paradigm for dataset distillation in depth.
We highlight two critical aspects:
1) The primary challenge is significant information loss, which hurts the quality and diversity in distilled images.
2) We empirically and theoretically demonstrate the critical attributes necessary for retaining the efficacy of \Relabel and seek an approach to effectively distill information from original images while adhering to these attributes.

\subsection{Preliminary}\label{sec:motivation}
\noindent{\bf Dataset distillation.}
Given a large-scale dataset $\cT= \{ \xx_{i}, y_{i} \}_{i=1}^{\abs{\cT}}$ which consists of ${\abs{\cT}}$ samples, dataset distillation aims to synthesize a smaller set $\cS=(\cS_X, \cS_Y) = \{ \widetilde{\xx}_{j}, \widetilde{y}_{j} \}_{j=1}^{\abs{\cS}}$ with ${\abs{\cS}}$ synthetic samples such that models trained on $\cT$ will have similar performance as models trained on $\cS$:
\begin{equation}
    \textstyle
    \mathbb{E}_{\xx\sim P_\mathcal{D}}[ \ell\left(\phi_{\mtheta_{\cT}}(\xx), y\right)]\simeq
    \mathbb{E}_{\xx\sim P_\mathcal{D}}[ \ell\left(\phi_{\mtheta_{\cS}}(\xx), y\right)] \,,
    \label{eq:ddobj}
\end{equation}
where $P_\mathcal{D}$ is the test real distribution, $\xx$ is a data sample, $\ell$ is the loss function, i.e., cross-entropy loss.
Here, $\mtheta_{\cT}$ and $\mtheta_{\cS}$  denotes the parameters of the neural network $\phi$ trained on $\cT$ and $\cS$, respectively.
\looseness=-1

\vspace{5pt}
\noindent{\bf A closer look at information extraction paradigm. }
As the first effective yet efficient solution to allow the dataset distillation on diverse-scale datasets such as ImageNet-1K~\cite{deng2009imagenet}, information extraction-based methods have attracted attention and inspired many subsequent works.
These methods typically employ a three-stage distillation process, \textit{indirectly} transferring information from the original to the distilled images.
The first stage involves condensing information from the complete dataset into pre-trained neural network models via \Squeeze, followed by the extraction of this information into distilled images using \Recover \cite{yin2023squeeze,liu2023dataset}.
Upon the distilled data samples, \Relabel applies a pre-trained model to further distill knowledge into the label space.
However, this paradigm encounters several key challenges:
\begin{enumerate}[leftmargin=12pt,itemsep=5pt]
    \item The \Recover stage \textit{necessitates batch normalization} in pre-trained models \cite{ioffe2015batch} to align the statistical features between distilled and original images \cite{yin2023squeeze,liu2023dataset}.
    \item \textit{Significant information loss} during both \Squeeze and \Recover stages leads to distilled images with minimal content, adversely affecting performance, particularly in low \IPC settings \cite{sun2024diversity}.
    \item Distilled images often \textit{exhibit unrealistic textures or semantics}, tailored to specific networks, thereby limiting their generalization ability \cite{shao2023generalized}.
    \item Despite outperforming other paradigms (e.g., matching-based methods) in terms of efficiency~\cite{yin2023squeeze}, the \Recover stage is still \textit{computationally demanding}, requiring numerous optimization iterations~\cite{yin2023dataset}.
\end{enumerate}

In the meanwhile, the influence of \Relabel is under-explored~\cite{yin2023squeeze,shao2023generalized}.
These challenges motivate us to explore a method to simultaneously address four key problems by \textit{directly} condensing information from the original dataset for image distillation, abandoning traditional decoupled \Squeeze and \Recover stages, and to further explore the role of \Relabel.

\begin{wrapfigure}{r}{0.5\linewidth} 
    \vspace{-25pt}
    \centering
   \includegraphics[width=\linewidth]{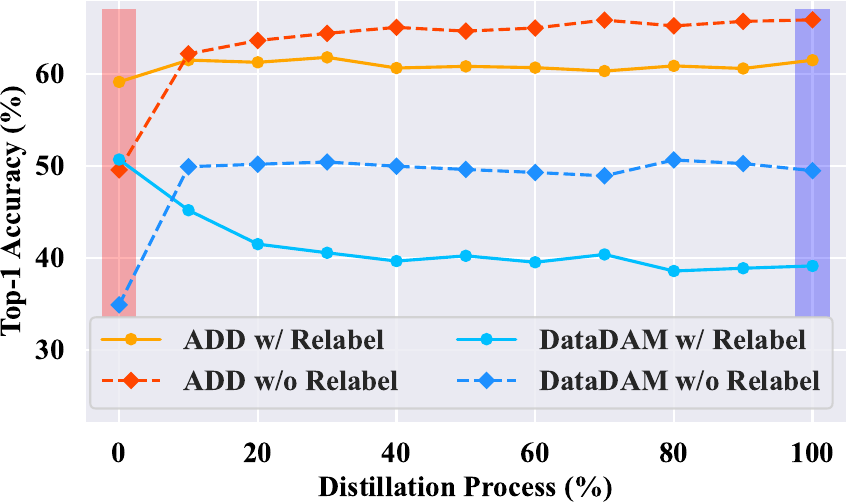}
   \vspace{-10pt}
    \caption{Applying \Relabel to ADD~\cite{zhang2023accelerating} and DataDAM~\cite{sajedi2023datadam}.
    We evaluate the distilled images with $\IPC=10$ during the distillation process.
    The results indicate that \Relabel only assists the early-stage distilled datasets.
    }
    \label{fig:dd_relabel}
    \vspace{-15pt}
\end{wrapfigure}

\subsection{Does \Relabel Always Help Distilled Images?}
\label{sec:relabel}
\Relabel has become a widely adopted and effective technique in dataset distillation, as demonstrated in recent works \cite{yin2023squeeze, sun2024diversity, shang2024gift}.
The core idea is to use a model pre-trained on the full real dataset to re-assign labels to the distilled images, and then train a student model on these relabeled distilled samples.

To evaluate the effect of \Relabel, we compare state-of-the-art DD methods with and without this strategy. 
As reported in \darkredtext{Tab. 7} (\darkredtext{App. E}), removing \Relabel leads to a \textit{substantial performance degradation}.
For example, on ImageNet-1k, the accuracy of SRe$^2$L and G-VBSM drops to 1.1\% and 0.8\%, respectively, when \Relabel is disabled.

Given its effectiveness, a natural question is whether \Relabel can be used as a plug-and-play component for DD methods that do not originally rely on it.
To answer this, we apply \Relabel to distilled images produced by representative non-\Relabel methods, such as ADD~\cite{zhang2023accelerating} and DataDAM~\cite{sajedi2023datadam}.
The results are shown in \figref{fig:dd_relabel}.
We observe that performance gains are mainly limited to the \emph{very early} distillation stage (i.e., when the ``distilled'' images remain close to the initial real images used for initialization).
In contrast, as distillation proceeds, relabeling the increasingly optimized distilled images provides little benefit and can even hurt performance.

This behavior suggests a mismatch between the relabeling model and the evolved distilled samples.
Specifically, the model used for \Relabel is trained exclusively on the original real dataset, whereas distilled images gradually deviate from the real-image distribution during optimization.
Such a distribution shift alters semantic and/or textural cues, making the pre-trained relabeler less reliable on distilled samples (see \darkredtext{App. K} for qualitative evidence).
Consequently, inaccurate relabels accumulate, and the downstream student model suffers.

\vspace{3pt}
\noindent{\bf A theoretical perspective.}
Beyond the above empirical observation, we theoretically reveal that a model well-trained on the original, undisturbed samples (i.e., the model used for \Relabel) may only provide suboptimal labels for shifted distilled samples.
Concretely, standard DD typically starts from a subset of real samples $\{\xx_j\}_{j=1}^{|\cS|}$ drawn from the full dataset $\cT$, and iteratively optimizes them into distilled samples $\{\widetilde{\xx}_j\}_{j=1}^{|\cS|}$ with $\widetilde{\xx}_j=\xx_j+\epsilon_j$, where $\epsilon_j$ denotes the learned perturbation.
A shift occurs because the optimization updates $\epsilon_j$ may change both semantics and textures of $\xx_j$, causing the distilled distribution to deviate from the original one.

\begin{proposition}
    \label{prop:subop_label}
    For two original Gaussian distributions $\cN_1(\mu_1, \sigma^2)$ and $\cN_2(\mu_2, \sigma^2)$, we first define their shifted versions $\cN_1^s(\mu_1 + s_1, \sigma^2)$ and $\cN_2^s(\mu_2 + s_2, \sigma^2)$. Then, the optimal classification model $f_{opt}$ for $\cN_1^s$ and $\cN_2^s$ achieves the following sub-optimal classification accuracy for $\cN_1$ and $\cN_2$:
    \begin{small}
        \begin{align}
            P_{\text{acc}} &= \frac{1}{2} \left( \Phi\left(\frac{\mu_2 + s_2 - \mu_1 + s_1}{2\sigma}\right) + 1 - \Phi\left(\frac{\mu_1 + s_1 - \mu_2 + s_2}{2\sigma}\right) \right) \\
            P_{\text{acc}} &\leq \frac{1}{2} \left( \Phi\left(\frac{\mu_2 - \mu_1}{2\sigma}\right) + 1 - \Phi\left(\frac{\mu_1 - \mu_2}{2\sigma}\right) \right)
        \end{align}
    \end{small}%
    Here, $\Phi(\cdot)$ is the cumulative distribution function (CDF) of the standard normal distribution.
\end{proposition}

That is, a relabeler trained on the original dataset may provide unreliable labels when applied to distribution-shifted distilled samples.
Therefore, to maximize the efficacy of \Relabel, it is necessary to ensure the distilled data distribution remains close to that of the original real dataset.
Only under such alignment can distilled samples be recognized and labeled correctly by the pre-trained relabeler.

\vspace{5pt}
\noindent{\bf Subset selection from real data.}
To achieve these goals, a straightforward approach involves directly selecting images from the original dataset to construct the distilled dataset.
Numerous works \cite{sun2024diversity,forgy1965cluster,welling2009herding,tran2025enhancing} have explored methods for selecting diverse and representative key samples from the original dataset to create a distilled dataset.
These methods can ensure that the distilled data can be accurately identified by the pre-trained model and maximize the effectiveness of \Relabel.
A notable contribution is RDED \cite{sun2024diversity}, which extracts key patches from each image based on \textit{high realism scores} to construct a distilled dataset.

\begin{figure}
    \centering
    \begin{subfigure}[b]{0.32\linewidth}
        \centering
        \includegraphics[width=\textwidth]{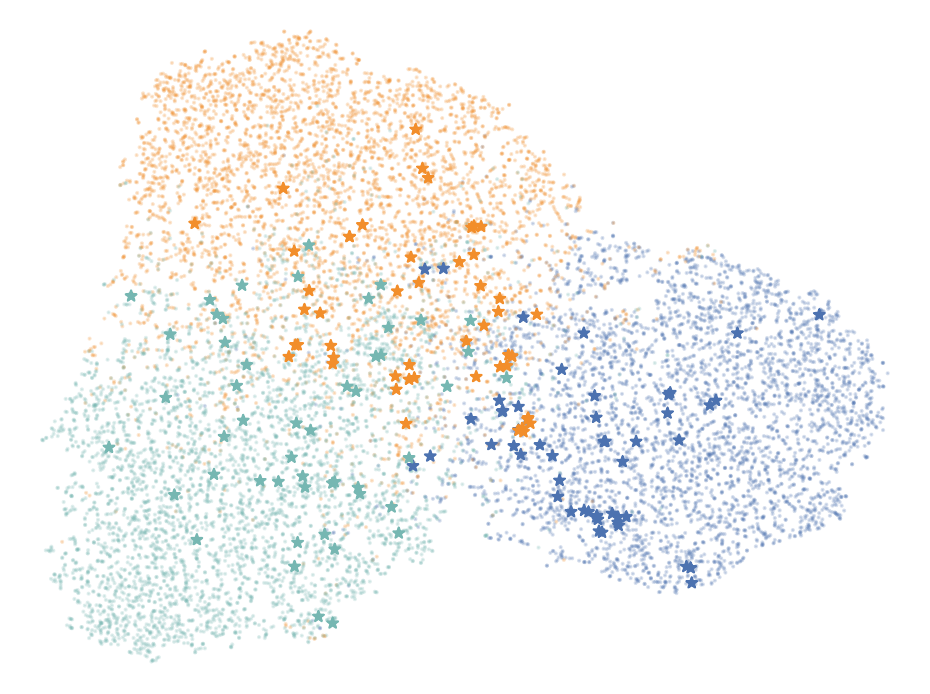}
        \subcaption{SRe$^2$L \cite{yin2023squeeze}}
    \end{subfigure}
    \begin{subfigure}[b]{0.32\linewidth}
        \centering
        \includegraphics[width=\linewidth]{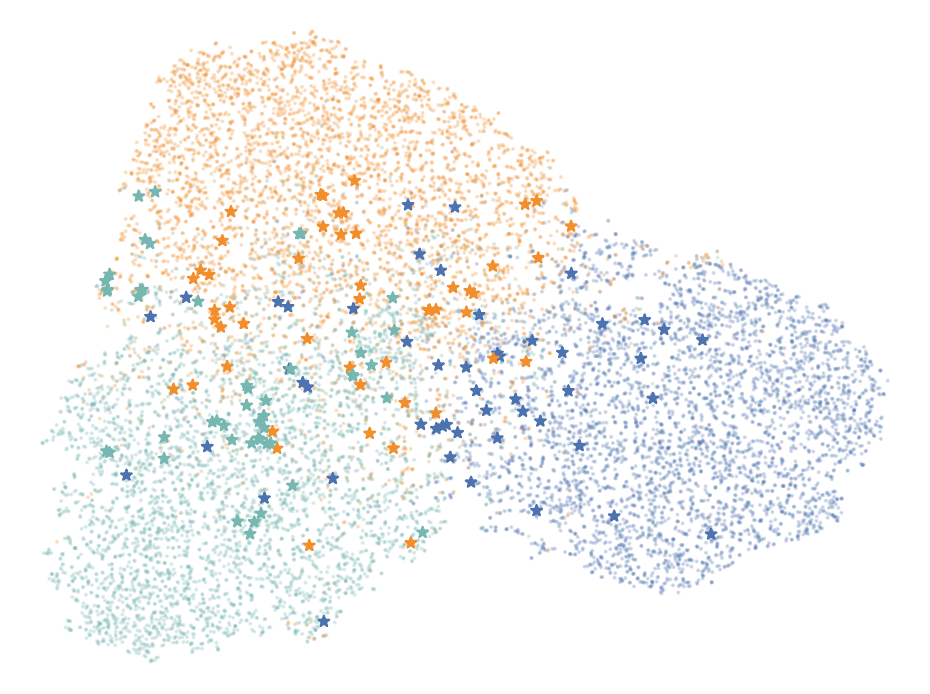}
        \subcaption{G-VBSM \cite{shao2023generalized}}
    \end{subfigure}
    \begin{subfigure}[b]{0.32\linewidth}
        \centering
        \includegraphics[width=\linewidth]{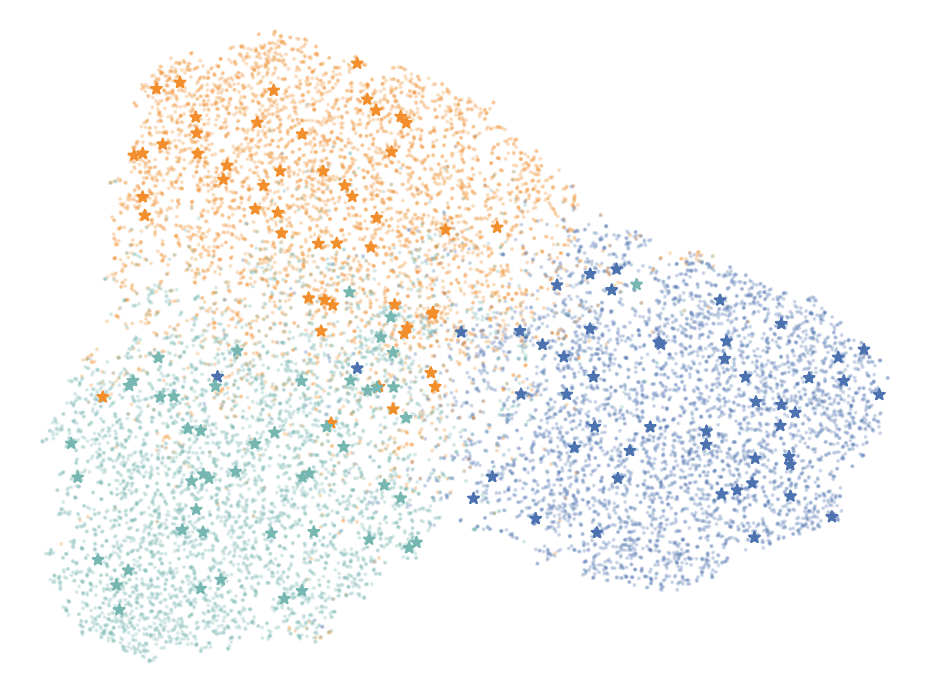}
        \subcaption{RDED \cite{sun2023diversity}}
    \end{subfigure}
    \vspace{-0.05in}
     \caption{Visualization of Feature Distributions for the Original and Distilled Datasets.
    The distilled datasets are optimized using state-of-the-art distillation techniques: SRe$^2$L \cite{yin2023squeeze}, G-VBSM \cite{shao2023generalized}, and RDED \cite{sun2023diversity}.
    \textcolor{orange}{Orange}, \textcolor{green}{green}, and \textcolor{blue}{blue} points depict the first three classes of CIFAR-10, while $\star$ points represent the corresponding distilled datasets with images per class (\IPC) of 50. The lighter shades represent the original dataset.}
   \label{fig:distribution}
      \vspace{-1em}
\end{figure}

To further validate whether the sample selected by RDED can achieve the distribution as the original dataset, we visualize the feature distributions \footnote{Features are extracted using a model trained on the full dataset, followed by visualization with T-SNE \cite{van2008visualizing}.} of the distilled dataset alongside those of the original dataset in \figref{fig:distribution}. Specifically, we compare the state-of-the-art methods, including SRe$^2$L \cite{yin2023squeeze}, G-VBSM \cite{shao2023generalized}, and RDED \cite{sun2024diversity}.
Among these, \textit{RDED demonstrates the closest alignment with the feature distribution of the original dataset.}
Additionally, as shown in \tabref{tb:main_cifar} and \tabref{tab:main_imagenet_resnet18}, RDED achieves superior performance compared to the other two SOTA methods.
Therefore, by default, \algopt adopts the selection mechanism of RDED\footnote{See \darkredtext{App. C} for details of RDED.}.
Note that the subset selection strategy itself is not the focus of this work.
Our proposed \algopt is selection-method-agnostic and can incorporate various selection strategies, as demonstrated in \tabref{tab:selection_ipc10}.

\section{Methodology}
\label{sec:method}

The following section begins by introducing the formal definitions of effective information for a given sample and the resulting information gap between any two given samples.
Furthermore, we propose a method based on minimizing the information gap between the selected sample subsets and the distilled samples to ensure the preservation of information integrity in the distilled set.
The distillation process of our \algopt is depicted in \figref{fig:method}, with the detailed algorithm outlined in \darkredtext{Alg. 1} in Appendix.

\vspace{5pt}
\noindent{\bf On the information gap of two samples.}
For the selected $\IPC$ subsets of key images, we aim to compress each of them into a more compact pixel space, i.e., distilled image $\widetilde{\xx}$, thus forming $\cS_X$.
However, given the constraints of limited pixel space storage, using a naive solution---e.g., directly resizing and concatenating multiple images from a subset into one---results in a significant reduction in the fineness and detail of the original images.
To effectively capture and condense the salient information from the original samples into the distilled ones, we aim to enable each distilled sample $\widetilde{\xx}$ to encapsulate the information of a selected subset from $\mathcal{T}_X$.
We begin by formally defining the effective information of a data sample as follows.

\begin{definition}[Observation-based Effective Information]
    \label{def:ei}
    Let $\xx_i$ represent a sample from any domain (e.g., image).
    Define an observer group $\cR = \{\xi_j\}$, where each observer $\xi_j$ is capable of extracting or interpreting features from $\xx_i$, and $|\cR| \geq 1$.
    The effective information of the sample $\xx_i$, as observed by the group $\cR$, is conceptualized as the distribution $\cP_{\xx_i | \cR}(z)$.
    This distribution is formulated as:
    \begin{equation}
        \cP_{\xx_i | \cR}(z) := \{z | z = \xi_j(\xx_i), \forall \xi_j \in \cR \} \,,
    \end{equation}
    where $z$ denotes the set of features or interpretations extracted from $\xx_i$ by an observer $\xi_j$ within $\cR$.
\end{definition}
The \defref{def:ei} posits that the effective information of a sample encompasses the set of its features as perceived or extracted by a diverse set of observers.
Samples are considered to have similar effective information if they result in comparable feature sets across the observers in $\cR$.
We consequently define the effective information gap between two given samples $\xx_i$ and $\xx_j$.

\begin{definition}[Pairwise Effective Information Gap]
    \label{def:ei_gap}
    Let $\xx_i$ and $\xx_j$ represent two samples from any domain, and given an observer group $\cR = \{\xi_j\}$.
    This effective information gap is formulated as the difference between their effective information:
    \begin{equation}
        \label{eq:ei_gap}
        I_G(\xx_i, \xx_j;\cR) = \textup{D}_{\textup{KL}}(\cP_{\xx_i | \cR} || \cP_{\xx_j | \cR}) \,.
    \end{equation}
\end{definition}

\begin{wrapfigure}{r}{0.6\linewidth} 
    \vspace{-25pt}
    \centering
   \includegraphics[width=\linewidth]{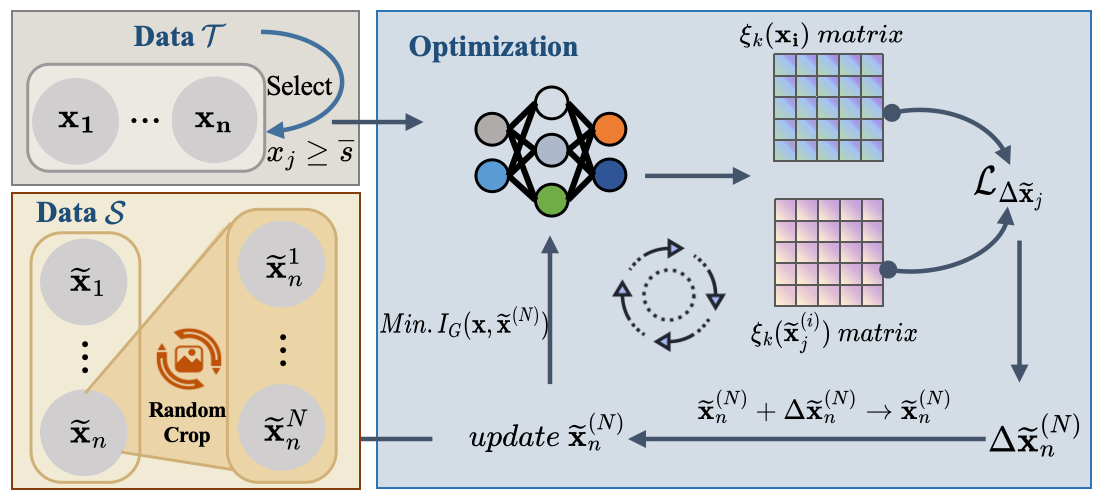}
   \vspace{-10pt}
    \caption{
    \textbf{Distillation Process of Our \algopt}.
    First, $\IPC$ subsets are selected from the original data $\cT$, where each subset contains images denoted as $\{ \xx_j \}_{j=1}^N$;
    Then, for each image $\widetilde{\xx}$ in the initial distilled data $\cS$, the \texttt{RandomCrop} is applied to generate views $\{ \widetilde{\xx}^n \}_{n=1}^N$.
    The information gap $I_G(\xx_j,  \widetilde{\xx}^n)$ is then minimized for each view. This process is iteratively performed for every distilled image $\widetilde{\xx}$;
    }
    \label{fig:method}
    \vspace{-15pt}
\end{wrapfigure}

\noindent{\bf Compressing information into distilled set.}
However, directly minimizing the information gap between distilled set $\cS_X$ and original set $\cD_X$ through~\defref{def:ei_gap} is intractable, due to \eqref{eq:ei_gap} is a sample-level metric that cannot be directly applied to the set.
Thanks to the widely adapted data augmentation techniques $\cA$ to enhance the diversity of each sample $\widetilde{\xx}$ in practice,
we relax the estimation of \eqref{eq:ei_gap} through calculating the effective information gap between a set of $N$ original samples $\{ \xx_i \}_{i=1}^N$ and a distilled sample $\widetilde{\xx}_j$ using $N$ augmented views of the distilled samples, i.e.,
\looseness=-1
\begin{equation}
    \label{eq:dis_goal}
    \begin{aligned}
        \E_{(\xx_i,\widetilde{\xx}_j^{(i)}) \sim (\{ \xx_i \}_{i=1}^N, \cA(\widetilde{\xx}_j))} I_G(\xx_i, \widetilde{\xx}_j^{(i)};\cR) 
        \textup{s.t.} \quad \cA(\widetilde{\xx}_j) = \{\widetilde{\xx}_j^{(1)}, \widetilde{\xx}_j^{(2)}, \ldots, \widetilde{\xx}_j^{(N)}\}.
    \end{aligned}
\end{equation}
However, we still cannot directly distill sample $\widetilde{\xx}_j$ through \eqref{eq:dis_goal} due to the intractable KL divergence estimation in \eqref{eq:ei_gap}, and thus we derive \thmref{thm:bound_ei_gap} that bounds \eqref{eq:ei_gap} to enable it computationally tractable.
\looseness=-1

\begin{theorem}[Pairwise Effective Information Gap]
    \label{thm:bound_ei_gap}
    Let $\xx_i$ and $\xx_j$ represent two samples from any domain, and given an observer group $\cR = \{\xi_j\}$.
    The effective information gap is upper-bounded as
    \looseness=-1
    \begin{equation}
        \begin{aligned}
            I_G(\xx_i, \xx_j;\cR) \leq \E_{\xi_k \sim \cR}{\| \xi_k(\xx_i) -  \xi_k(\xx_j) \|^2} 
            \textup{s.t.} \quad k \in [1, |\cR|] \,.
        \end{aligned}
    \end{equation}    
\end{theorem}
Therefore, we can capture a distilled sample $\widetilde{\xx}_j$ through combining \eqref{eq:dis_goal} and \thmref{thm:bound_ei_gap}, namely,
\begin{equation}\small
    \label{eq:sobj}
    \resizebox{.9\hsize}{!}{$
        \argmin_{\widetilde{\xx}_j} \E_{(\xx_i,\widetilde{\xx}_j^{(i)}) \sim (\{ \xx_i \}_{i=1}^N, \cA(\widetilde{\xx}_j))}\E_{\xi_k \sim \cR}{\| \xi_k(\xx_i) -  \xi_k(\widetilde{\xx}_j^{(i)}) \|^2}
    $}
\end{equation}

To reduce computational complexity, we utilize \texttt{RandomCrop} to generate the augmentations $\cA(\widetilde{\xx}_j)$ and concatenate resized real images to initialize distilled sample $\widetilde{\xx}_j$. We further consider a specific scenario where the observer group consists of only one pre-trained model across various transformations\footnote{
    In the context of image data, these transformations encompass a range of nonlinear and linear operations, including techniques for image data augmentation.
}.
This strategy allows us to employ multiple observers without requiring an excessive number of pre-trained models.
\looseness=-1
Our loss function is defined as follows to achieve \eqref{eq:sobj}:
\begin{equation}
    \label{eq:loss}
    \begin{aligned}
        \cL_{\Delta \widetilde{\xx}_j} = \E_{(\xx_i,\widetilde{\xx}_j^{(i)}) \sim (\{ \xx_i \}_{i=1}^N, \cA(\widetilde{\xx}_j+ \Delta \widetilde{\xx}_j))}
        \E_{\zeta_k \sim \cG} \left\| \zeta_k \circ \phi_{\mtheta_\cT}(\xx_i) -  \zeta_k \circ \phi_{\mtheta_\cT}(\widetilde{\xx}_j^{(i)}) \right\|^2
    \end{aligned}
\end{equation}
where $\cR = \{ \xi_k \mid \xi_k = \zeta_k \circ \phi_{\mtheta_\cT} , \forall \zeta_k \sim \cG \}$ and $\cG$ denotes the transformation group, $\cA(\widetilde{\xx}_j + \Delta \widetilde{\xx}_j)=\{\widetilde{\xx}_j^{(1)}, \widetilde{\xx}_j^{(2)}, \ldots, \widetilde{\xx}_j^{(N)}\}$.
By minimizing \eqref{eq:loss}, we find the optimal $\Delta \widetilde{\xx}_j^\star$ and capture the image $\widetilde{\xx}_j \gets \widetilde{\xx}_j + \Delta \widetilde{\xx}_j^\star$ as the distilled image.

\vspace{5pt}
\noindent{\bf Balancing the semantic and textural information.}
Directly generating the distilled image $\widetilde{\xx}_j$ through aligning its effective information with the original image set $\{ \xx_i \}_{i=1}^N$ in the last layer of model $\phi_{\mtheta_\cT}$ may lead to a substantial loss of texture information.
The reason is that the model $\phi_{\mtheta_\cT}$ tends to extract semantic information from input images, which often results in a notable drop in the texture details.
To balance between semantic richness and texture preservation in the distilled image $\widetilde{\xx}_j$, we leverage intermediate model features instead of the last-layer logits, which helps in retaining more textural details (see Section~\ref{sec:featalign} for the validation of its robustness).
\looseness=-1

\vspace{5pt}
\noindent {\bf \Relabel across Various Transformations.}
\label{sec:relabel}
We propose an improved re-labeling strategy for our informative distilled images $\widetilde{\xx}$, which provides more diverse and informative knowledge in the label space compared to the basic one-shot labeling approach. The standard \Relabel technique~\cite{yin2023squeeze}, inspired by~\cite{yun2021re}, suggests that a random image crop may contain a different object than the one originally labeled, leading to inaccurate or misleading training data. This highlights the limitations of the one-shot labeling strategy in expressing sufficient knowledge for cropped images.

Our extended \Relabel approach considers a wider range of image transformations beyond random cropping. Similar to the soft labeling approach in \cite{shen2022fast}, we generate transformed-view-level soft label $\widetilde{y}_k = \phi_{\mtheta_\cT} ( \zeta_k (\widetilde{\xx}) )$, where $\zeta_k (\widetilde{\xx})$ represents the $k$-th transformation applied to the distilled image $\widetilde{\xx}$.
\looseness=-1

Therefore, we can train the model $\phi_{\mtheta_\cS}$ on the distilled data by minimizing:
\begin{equation}
    \textstyle
    \cL = - \sum_{j}{\sum_{k}{\| {\phi_{\mtheta_\cS}( \zeta_k(\widetilde{\xx}_{j}) )}}} - \widetilde{y}_{(j,k)} \|^2 \,.
\end{equation}

The whole distillation process for an entire dataset by using our \algopt is illustrated in \darkredtext{Alg. 1} in Appendix.

\section{Experiment}\label{sec:exp}
In this section, we evaluate the performance of our proposed \algopt over various datasets and neural architectures.
First, we demonstrate the superior results of \algopt on real-world datasets, cross-architecture generalization and efficiency.
We next perform comprehensive ablation studies to evaluate the impact of each component in our proposed method, as well as to analyze the influence of hyperparameter choices and subset selection strategies. Finally, we demonstrate the superior performance of our approach, \algopt, in continual learning applications.

\subsection{Experimental Setting} \label{sec:expset}

\noindent{\bf Datasets and neural network architectures.}
We conduct experiments on varying scales and resolutions of images.
\begin{itemize}[leftmargin=12pt]
    \item \textbf{Small-scale:} we evaluate on two datasets, including CIFAR-10 ($32 \times 32$) \cite{krizhevsky2009cifar} and CIFAR-100 ($32 \times 32$) \cite{krizhevsky2009learning}.
    \item \textbf{Large-scale:} we also use two large-scale high-resolution datasets including Tiny-ImageNet ($64 \times 64$) \cite{le2015tiny} and ImageNet-1K ($224 \times 224$) \cite{deng2009imagenet}.
\end{itemize}
Following prior dataset distillation works \cite{yin2023squeeze,zhao2023improved,guo2023towards}, we employ  ConvNet \cite{guo2023towards}, ResNet-18 \cite{he2016deep}, and MobileNet-V2 \cite{sandler2018mobilenetv2}, ViT-T/16 \cite{dosovitskiy2021image}, ShuffleNet-v2-x2.0 \cite{ma2018shufflenet}, DenseNet-121 \cite{huang2018densely}, as our backbone networks.
Specifically, for ConvNet, we use Conv-3 on CIFAR-10/100 and use Conv-4 on Tiny-ImageNet and ImageNet-1K.
More details about the used datasets and architectures can be found in \darkredtext{App. G}.

\vspace{5pt}
\noindent{\bf Baselines.}
We compare our method with several SOTA distillation methods that can scale to large high-resolution datasets, including G-VBSM \cite{shao2023generalized}, SRe$^2$L \cite{yin2023squeeze}, RDED \cite{sun2024diversity}, CDA \cite{yin2023dataset},  WMDD \cite{liu2024dataset}, Teddy \cite{yu2024teddy}, CUDD \cite{du2024diversity}, EDC \cite{shao2024elucidating}, DWA~\cite{du2024diversity}, CV-DD~\cite{cui2025dataset}, INFER~\cite{zhang2024breaking}, GIFT \cite{shang2024gift}, NRR-DD \cite{tran2025enhancing}, DELT \cite{shen2024delt} .
The results of additional baseline methods, including DataDAM \cite{sajedi2023datadam}, ADD \cite{zhang2023accelerating}, IDM \cite{zhao2023improved}, CDA \cite{yin2023dataset}, WMDD \cite{liu2024dataset}, DATM \cite{guo2023towards}, DREAM \cite{liu2023dream}, and FreD \cite{shin2023frequency}, are presented in \darkredtext{App. G}. Comprehensive details regarding these approaches are also provided in the same appendix.

\vspace{5pt}
\noindent{\bf Implementation details.}
By default, \algopt employs the subset selection mechanism of RDED. Notably, the proposed \algopt is selection-method-agnostic and can seamlessly integrate alternative selection strategies.
All distilled datasets synthesized from these baselines are evaluated using the same post-training process. 
All the hyper-parameters used in our \darkredtext{Alg. 1} are general, insensitive, and easy-implemented for all datasets and network architectures (c.f.\ ~\secref{sec:ablationstudy} and \darkredtext{App. I} for validation).
We employ a generalized configuration for $\cT^\prime$ (c.f.\ \darkredtext{App. C} for definition), where the size of subset $|\cT^\prime|$ is set as $300$.
We set the number $N=4$ of images squeezed in a distilled image and number $M=200$ of compression iteration (c.f.\ ~\darkredtext{Alg. 1} for definition).
More implementation details are provided in \darkredtext{App. G}.

\begin{table*}[!t]
    \centering
      \caption{\small
        \textbf{Comparison with baseline approaches on CIFAR-10 (``CF-10'') and CIFAR-100 (``CF-100'').}
        In the table, \textbf{bold} indicates the best result.
        \underline{Underline} means the second-best result.
        \IPC refers to the Images Per Class within distilled datasets. 
    }
    \vspace{-8pt}
    \label{tb:main_cifar}
    \setlength\tabcolsep{2pt}
    \resizebox{1.\textwidth}{!}{
        \begin{tabular}{@{}cc|cccc|cccc|cccc@{}}
            \toprule[1.5pt]
            \multicolumn{2}{c|}{}   & \multicolumn{4}{c|}{ConvNet} & \multicolumn{4}{c|}{ResNet-18} & \multicolumn{4}{c}{ResNet-50}                                                                                           \\ \midrule
            Dataset                 & IPC                          & G-VBSM                         & SRe2L                         & RDED                    & \algopt (Ours)          & G-VBSM         & SRe2L          & RDED           & \algopt (Ours)          & G-VBSM         & SRe2L          & RDED                   & \algopt (Ours)          \\ \midrule
            \multirow{3}{*}{CF-10}  & 1                            & 21.2 $\pm$ 0.2                 & 22.2 $\pm$ 1.1                & \underline{28.9 $\pm$ 0.4}          & \textbf{37.4 $\pm$ 0.4} & 17.0 $\pm$ 1.0 & 19.9 $\pm$ 0.9 & \underline{27.8 $\pm$ 0.4} & \textbf{32.3 $\pm$ 0.2} & 17.2 $\pm$ 0.9 & 20.2 $\pm$ 0.5 & \underline{24.8 $\pm$ 0.5}         & \textbf{31.8 $\pm$ 0.7} \\
                                    & 10                           & 38.6 $\pm$ 0.8                 & 39.6 $\pm$ 0.4                & \underline{56.0 $\pm$ 0.1}          & \textbf{61.4 $\pm$ 0.4} & 36.3 $\pm$ 0.7 & 39.4 $\pm$ 0.9 & \underline{47.3 $\pm$ 0.5} & \textbf{66.2 $\pm$ 0.9} & 33.8 $\pm$ 1.1 & 37.2 $\pm$ 0.6 & \underline{45.1 $\pm$ 0.7}         & \textbf{62.9 $\pm$ 0.3} \\
                                    & 50                           & 62.7 $\pm$ 0.5                 & 57.7 $\pm$ 0.4                & \underline{71.1 $\pm$ 0.2}          & \textbf{74.7 $\pm$ 0.2} & 64.5 $\pm$ 0.6 & 62.8 $\pm$ 1.2 & \underline{76.4 $\pm$ 0.4} & \textbf{85.1 $\pm$ 0.3} & 61.5 $\pm$ 0.6 & 61.6 $\pm$ 0.2 & \underline{74.1 $\pm$ 0.6}         & \textbf{84.2 $\pm$ 0.3} \\ \midrule
            \multirow{3}{*}{CF-100} & 1                            & 13.4 $\pm$ 0.3                 & 12.9 $\pm$ 0.1                & \underline{21.8 $\pm$ 0.4}          & \textbf{27.4 $\pm$ 0.4} & \underline{13.4 $\pm$ 0.5} & 11.5 $\pm$ 0.4 & 4.6 $\pm$ 0.1  & \textbf{31.1 $\pm$ 0.7} & \underline{12.6 $\pm$ 0.6} & 10.1 $\pm$ 0.1 & 4.5 $\pm$ 0.2          & \textbf{28.1 $\pm$ 0.2} \\
                                    & 10                           & 38.7 $\pm$ 0.8                 & 34.2 $\pm$ 0.3                & \underline{47.0 $\pm$ 0.3}          & \textbf{49.3 $\pm$ 0.3} & 47.0 $\pm$ 0.4 & 42.7 $\pm$ 0.5 & \underline{53.4 $\pm$ 0.3} & \textbf{61.7 $\pm$ 0.2} & 47.5 $\pm$ 0.5 & 44.2 $\pm$ 0.5 & \underline{54.0 $\pm$ 0.3}         & \textbf{62.9 $\pm$ 0.4} \\
                                    & 50                           & 53.8 $\pm$ 0.4                 & 52.2 $\pm$ 0.3                & \textbf{55.3 $\pm$ 0.2} & \underline{55.1 $\pm$ 0.2}          & 60.0 $\pm$ 0.1 & 57.4 $\pm$ 0.2 & \underline{64.0 $\pm$ 0.0} & \textbf{67.0 $\pm$ 0.1} & 62.2 $\pm$ 0.3 & 60.6 $\pm$ 0.2 & \underline{65.8 $\pm$ 0.3}         & \textbf{67.9 $\pm$ 0.1} \\ \bottomrule[2pt]
        \end{tabular}        }
\end{table*}

\subsection{Comparison with the SOTA Methods}
\label{sec:comparesota}

\noindent{\bf Small-scale datasets.}
Following previous research~\cite{cazenavette2022dataset,cui2023scaling,zhao2023improved}, we set \IPC to 1, 10, and 50 to compare with baselines on varying datasets and networks.
As the results reported in \tabref{tb:main_cifar}, \textit{our method \algopt outperforms other methods on varying datasets and neural networks with different \IPC.}
It is noteworthy that prior information extraction-based solutions like SRe$^2$L struggle in scenarios involving small distilled datasets such as CIFAR-100 with $\IPC = 1$ or CIFAR-10 with all \IPC, further verifying our claims proposed in \secref{sec:motivation}.
Detailed comparison among more datasets and baselines can be found in \darkredtext{App. H}.
Furthermore, we visualize the distilled images of various methods in \darkredtext{App. K}.

\begin{table*}[!t]
\centering
\caption{\small Comparison with the state-of-the-art dataset distillation methods on Tiny-ImageNet (``TN-IN'') and ImageNet-1K (``IN-1K'') on ResNet-18.}
\vspace{-8pt}
\setlength\tabcolsep{5.2pt}
\resizebox{1\textwidth}{!}{
\begin{tabular}{@{}lc|cccccccccc@{}}
\toprule[2pt]
Dataset& IPC & SRe2L & G-VBSM  & CDA & WMDD & RDED & Teddy & GIFT & NRR-DD &  DELT & \algopt(Ours)\\
\midrule
\multirow{3}{*}{TN-IN}
&1  &13.5 $\pm$ 0.2 &9.0 $\pm$ 0.1  &- &7.6 $\pm$ 0.2 &15.4 $\pm$ 0.6 & -& \underline{15.9 $\pm$ 0.3} & 13.5 $\pm$ 0.2 &  9.3 $\pm$ 0.5 &\textbf{25.1 $\pm$ 0.4}\\
&10&43.6 $\pm$ 0.5 &37.7 $\pm$ 0.3 &- &41.8 $\pm$ 0.1 &48.4 $\pm$ 0.3 & - &  \underline{49.2 $\pm$ 0.1}  & 45.2 $\pm$ 0.2 & 43.0 $\pm$ 0.1 &\textbf{53.3$\pm$0.1}\\
&50&53.4 $\pm$ 0.3 &52.2 $\pm$ 0.0 &48.7 $\pm$ 0.1 &\underline{59.4 $\pm$ 0.5} &57.4 $\pm$ 0.2 & 55.2 $\pm$ 0.1 & 58.1 $\pm$ 0.1 & \textbf{61.2 $\pm$ 0.1} & 55.7 $\pm$ 0.5 &58.6 $\pm$ 0.1\\
\midrule
\multirow{4}{*}{IN-1K}
&1  &2.8 $\pm$ 0.2  &2.1 $\pm$ 0.1  &- &3.2 $\pm$ 0.3 &5.9 $\pm$ 0.1 & - & - & \textbf{11.6 $\pm$ 0.2} & - &\underline{7.3$\pm$0.3}\\
&10&31.1 $\pm$ 0.1 &35.7 $\pm$ 0.2 &- &38.2 $\pm$ 0.2 &41.1 $\pm$ 0.2 &34.1 $\pm$ 0.1 &43.2 $\pm$ 0.1 &\underline{46.1 $\pm$ 0.0} &45.8 $\pm$ 0.0 &\textbf{48.7$\pm$0.2}\\
&50&49.5 $\pm$ 0.2 &51.6 $\pm$ 0.2 &53.5 $\pm$ 0.1 &57.6 $\pm$ 0.5 &55.3 $\pm$ 0.2 &52.5 $\pm$ 0.1 &56.5 $\pm$ 0.1 & \underline{60.1 $\pm$ 0.1} &59.2 $\pm$ 0.1 &\textbf{60.4$\pm$0.1}\\
&100&54.3 $\pm$ 0.2 &56.1 $\pm$ 0.1 &58.0 $\pm$ 0.1 &\underline{60.7 $\pm$ 0.0} &58.6 $\pm$ 0.1 &56.5 $\pm$ 0.0 &59.3 $\pm$ 0.2 & - & \textbf{62.4 $\pm$ 0.1} &\textbf{62.4$\pm$0.1}\\
\bottomrule[2pt]
\end{tabular}
}
\label{tab:main_imagenet_resnet18}
\vspace{-10pt}
\end{table*}

\vspace{3pt}
\noindent{\bf Large-scale datasets.}
To evaluate the practicality of \algopt in real-world settings, we further conduct experiments on the large-scale and high-resolution benchmarks Tiny-ImageNet and ImageNet-1K. The corresponding results are reported in \tabref{tab:main_imagenet_resnet18}.
In addition, we compare \algopt with several state-of-the-art methods on ResNet-50, as shown in \tabref{tb:main_imagenet_resnet50}.
Note that we exclude some methods listed in \tabref{tab:main_imagenet_resnet18} from the ResNet-50 comparison, since their distilled datasets are not available on larger networks, making a fair reproduction infeasible.
Furthermore, beyond the commonly used setting of \IPC=50, we also report results under larger budgets (e.g., \IPC=100) to assess scalability.
It is obvious that our \algopt achieves the best performance on most scenarios, demonstrating the effectiveness.

\begin{table}[!t]
    \centering
    \caption{\small Comparison with baselines on ImageNet-1K using ResNet-50.}
\vspace{-8pt}
    \label{tb:main_imagenet_resnet50}
    \renewcommand{\arraystretch}{1}
        \setlength\tabcolsep{2pt}
    \resizebox{\linewidth}{!}{%

        \begin{tabular}{@{}l|ccccccccccc@{}}
            \toprule[1.5pt]
            IPC & G-VBSM & SRe2L & RDED & Teddy & CUDD & CDA & EDC & DWA & INFER & CV-DD & \algopt\ (Ours) \\
            \midrule
            10  & 42.6 $\pm$ 0.4 & 38.3 $\pm$ 0.5 & 46.2 $\pm$ 0.3 & 37.2 $\pm$ 0.1 & 46.2 $\pm$ 0.3 & -- & \underline{54.1 $\pm$ 0.2} & 43.0 $\pm$ 0.1 & 38.3 $\pm$ 0.2 & 51.3 $\pm$ 0.2 & \textbf{54.3 $\pm$ 0.4} \\
            50  & 60.3 $\pm$ 0.1 & 58.4 $\pm$ 0.1 & 62.5 $\pm$ 0.1 & 58.5 $\pm$ 0.1 & 63.6 $\pm$ 0.2 & 61.3 $\pm$ 0.2 & \underline{64.3 $\pm$ 0.2} & 62.3 $\pm$ 0.3 & 63.4 $\pm$ 0.2 & 63.9 $\pm$ 0.1 & \textbf{65.9 $\pm$ 0.1} \\
            100 & 64.1 $\pm$ 0.1 & 62.9 $\pm$ 0.1 & 65.5 $\pm$ 0.1 & -- & \underline{66.7 $\pm$ 0.1} & 65.1 $\pm$ 0.1 & -- & 65.7 $\pm$ 0.1 & -- & -- & \textbf{67.6 $\pm$ 0.1} \\
            \bottomrule[2pt]
        \end{tabular}%
    }
\end{table}

\vspace{3pt}
\noindent{\bf Cross-architecture generalization.}
\label{sec:crossarch}
An important property of the distilled datasets is their good generalization capability across unseen architectural models.
Here we evaluate the generalizability of our distilled datasets when \IPC=10.
As reported in \tabref{tb:crossarch}, \textit{our distilled dataset achieves the best performance on unseen networks}, which reflects the good generalizability of the data and labels distilled by our method.
Our success stems from our \algopt effectively keeps both textural and semantic information in distilled images.

\begin{table*}[!t]
    \setlength\tabcolsep{4.5pt}
    \centering
    \caption{\small
        \textbf{Top-1 accuracy (\%) on ImageNet-1k for cross-architecture generalization.}
        We utilize ResNet-18 for distilling the original dataset. Subsequently, the distilled data is transferred to other architectures, with $\IPC=10$.
    }
    \vspace{-8pt}
    \label{tb:crossarch}
    \resizebox{.95\textwidth}{!}{
        \begin{tabular}{@{}c|cccccc@{}}
            \toprule
            Verifier & EfficientNet-B0         & ResNet-18                          & MobileNet-V2            & ViT-T/16                & ShuffleNet-v2-x2.0      & DenseNet-121            \\ \midrule
            SRe$^2$L & 35.2 $\pm$ 0.0          & 35.0 $\pm$ 0.3                     & 27.6 $\pm$ 0.0          & 3.2 $\pm$ 0.0           & 38.4 $\pm$ 0.3          & 43.2 $\pm$ 0.1          \\
            G-VBSM   & 31.5 $\pm$ 0.0          & \multicolumn{1}{l}{30.2 $\pm$ 1.2} & 26.0 $\pm$ 0.4          & 3.3 $\pm$ 0.1           & 31.8 $\pm$ 0.9          & 39.6 $\pm$ 0.1          \\
            RDED     & 41.1 $\pm$ 0.0          & \multicolumn{1}{l}{38.4 $\pm$ 0.7} & 34.8 $\pm$ 0.0          & 8.5 $\pm$ 0.1           & 43.2 $\pm$ 0.3          & 47.4 $\pm$ 0.3          \\
            Ours     & \textbf{48.5 $\pm$ 0.2} & \textbf{48.7 $\pm$ 0.3}            & \textbf{42.3 $\pm$ 0.2} & \textbf{10.8 $\pm$ 0.0} & \textbf{50.8 $\pm$ 0.2} & \textbf{54.4 $\pm$ 0.2} \\ \bottomrule[1.5pt]
        \end{tabular}
    }
\vspace{-10pt}
\end{table*}

\vspace{3pt}
\noindent{\bf Efficiency comparison.}
\label{sec:efficiencycomp}
Efficiency is also a key factor during the distillation process.
Here, we use a single RTX-4090 GPU for two methods to conduct experiments on Tiny-ImageNet.
The reason why we mainly compare with SRe$^2$L and G-VBSM is based on their outstanding as efficient optimization-based methods currently.
We evaluate the distillation efficiency by recording the run-time cost and peak GPU memory usage of distilling the image.
As evidenced in \tabref{tb:efficiency}, \textit{our \algopt achieves superior efficiency in comparison to SOTA methods}, with the exception of RDED, which benefits from an optimization-free paradigm~\cite{sun2024diversity}, demonstrating a notable advantage of efficacy and efficiency.
Significantly, our algorithm can offer a versatile peak memory capacity, enabling adjustments to batch size dynamically without sacrificing performance.
This efficiency is attributed to the fact that our \darkredtext{Alg. 1} can independently\footnote{
    Traditional distillation techniques necessitate the concurrent synthesis of a set of images to ensure collective quality \cite{yin2023squeeze,cazenavette2022dataset,guo2023towards}.
    This issue stems from the nature of these methods, which involve using a distilled dataset to approximate the original dataset with a similarity metric.
} optimize images, allowing us to distill them one by one.
More comparisons are in \darkredtext{App. H}.

\begin{figure*}[!t]
    \centering
    \begin{subfigure}{0.22\linewidth}
        \includegraphics[width=\linewidth, trim=0 0 0 0, clip]{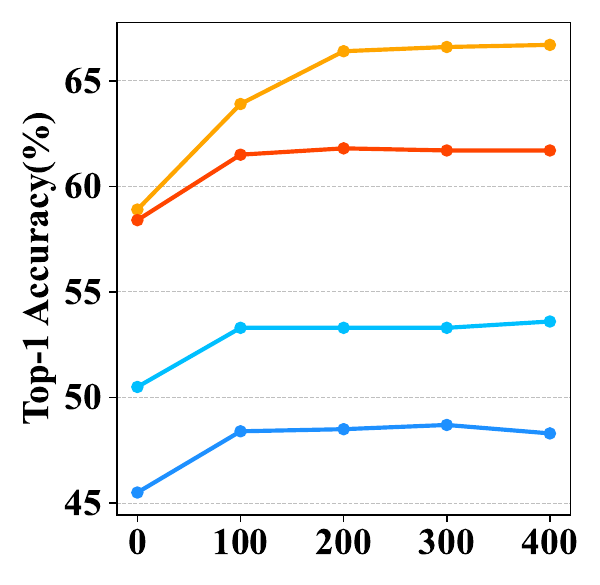}
        \caption{Number $M$}
        \label{fig:iteration}
    \end{subfigure}
    \hfill
    \begin{subfigure}{0.22\linewidth}
        \includegraphics[width=\linewidth, trim=0 0 0 0, clip]{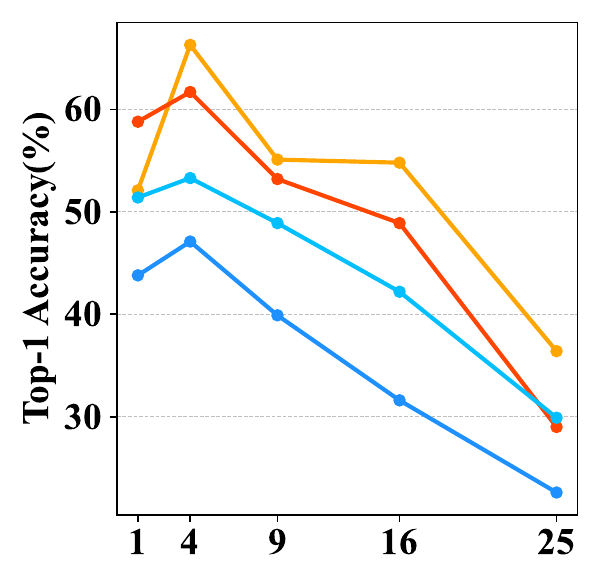}
        \caption{Number $N$}
        \label{fig:squeezenipc10}
    \end{subfigure}
    \hfill
    \begin{subfigure}{0.22\linewidth}
        \includegraphics[width=\linewidth, trim=0 0 0 0, clip]{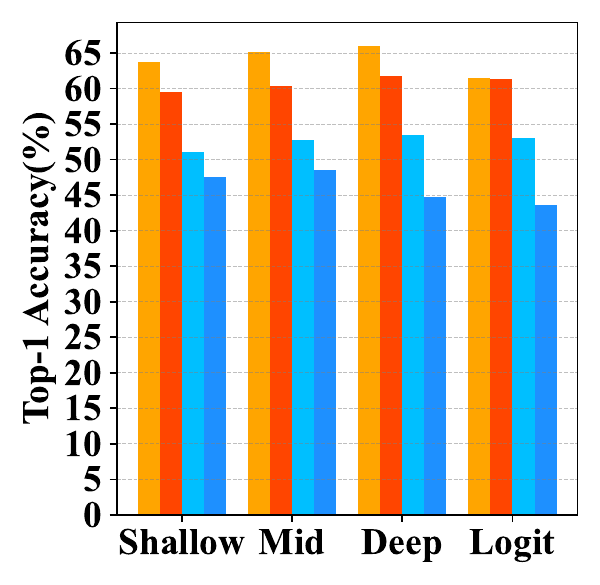}
        \caption{Alignment Layer}
        \label{fig:featurelayer}
    \end{subfigure}
    \hfill
    \begin{subfigure}{0.22\linewidth}
        \includegraphics[width=\linewidth, trim=0 0 0 0, clip]{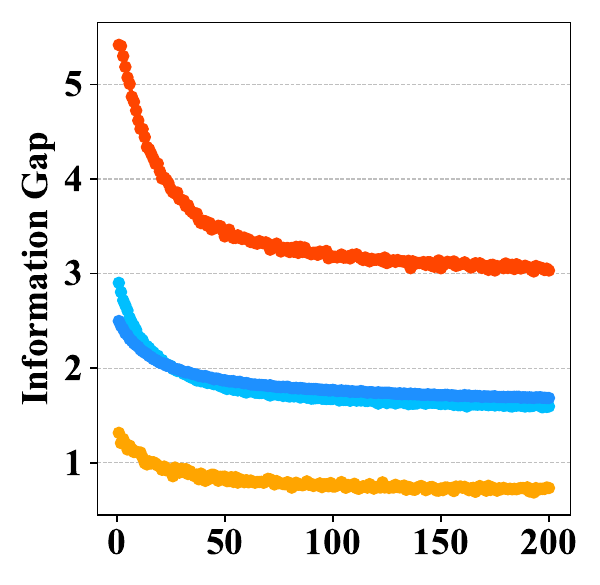}
        \caption{Number $K$}
        \label{fig:Loss}
    \end{subfigure}
    \vspace{-0.5em}
    \caption{\small
        \textbf{Ablation study on each component in our \algopt.}
        We evaluate our \algopt with different number $M$ of compression iterations~(\ref{fig:iteration}), number $N$ of images squeezed in one distilled image~(\ref{fig:squeezenipc10}), feature alignment layer~(\ref{fig:featurelayer}), and the information gap with respect to the number $K$ of iterations(\ref{fig:Loss}).
        The yellow \textcolor{yellow}{$\bullet$}, red \textcolor{red}{$\bullet$}, blue \textcolor{cyan}{$\bullet$} and deep blue \textcolor{blue}{$\bullet$} denote CIFAR-10, CIFAR-100, Tiny-ImageNet and ImageNet-1k respectively.
    }
    \vspace{-8pt}
    \label{fig:peroci}
\end{figure*}

\begin{table*}[!t]
    \centering
        \caption{\small
        \textbf{Efficiency comparison on varying networks.}
        Following SRe$^2$L, Time Cost is the consumption (s) when generating 100 images simultaneously, and the peak value of GPU memory usage is measured with a batch size of 100.
    }
    \label{tb:efficiency}
    \vspace{-8pt}
    \setlength\tabcolsep{7pt}
    \scalebox{0.8}{
        \begin{tabular}{@{}c|cccc|cccc@{}}
            \toprule
            \multirow{2}{*}{Architecture} & \multicolumn{4}{c|}{Time Cost (s)} & \multicolumn{4}{c}{Peak Memory (GB)}                                                   \\ \cmidrule(l){2-9}
                                          & SRe$^2$L                           & G-VBSM                               & RDED  & Ours  & SRe$^2$L & G-VBSM & RDED & Ours \\ \midrule
            Conv-4                        & 51.68                              & 259.84                               & 1.682 & 13.02 & 1.36     & 4.94   & 0.74 & 0.65 \\
            ResNet-18                     & 191.14                             & 259.84                               & 2.78  & 25.34 & 3.62     & 4.94   & 0.92 & 1.56 \\
            MobileNet-V2                  & 114.05                             & 259.84                               & 4.07  & 18.81 & 1.27     & 4.94   & 0.29 & 0.64 \\ \bottomrule[1.5pt]
        \end{tabular}
    }
    \vspace{-8pt}
\end{table*}

\vspace{-10pt}
\subsection{Ablation Study} \label{sec:ablationstudy}
In this section, we set the default $\IPC = 10$ and employ ConvNet as the network backbone to examine how the components used in our \algopt influence the quality of distilled dataset (see \darkredtext{App. I} for more investigation).

\vspace{3pt}
\noindent{\bf Influence of compression iteration number $M$.}
The number of distillation iterations, denoted as $M$, impacts two aspects:
1) A higher $M$ enhances \algopt's ability to generate higher-quality images;
2) A lower $M$ ensures a faster execution of our \darkredtext{Alg. 1}.
Consequently, choosing an optimal iteration number $M$ represents a balance between quality and speed.
As illustrated in \figref{fig:iteration}, an iteration count of $M=200$ offers a well-rounded compromise for various datasets.
Additionally, it is noteworthy that \textit{our \algopt exhibits robustness to variations in $M$.} Specifically, setting $M$ beyond 200 yields negligible differences in performance.

\vspace{3pt}
\noindent{\bf Influence of size of compressed images $N$.}
Though we can compress more original images $\{\xx_i\}_{i=1}^N$ into each distilled image $\widetilde{\xx}$ by increasing $N$ to benefit the feature diversity (c.f. \secref{sec:method}), it also results in less information preserved from each original image $\xx_i$.
\figref{fig:squeezenipc10} showcase that \textit{the validation performance rises to the highest on selected four datasets when $N=4$.}

\vspace{3pt}
\noindent{\bf Why and how to choose the feature alignment layer?}
\label{sec:featalign}
The experimental results in \figref{fig:featurelayer} illustrate the impact of the feature alignment layer, alongside the discussion in \secref{sec:method}.
As depicted in \figref{fig:featurelayer},  \textit{high performance is often achieved through alignment at the middle layer.}
This outcome likely stems from the varied information encoded at different network depths: shallow layers capture more textural details, whereas logit and deep layers are more adept at encoding semantic information.
Thus, aligning at the middle layer enables the distilled dataset to achieve an optimal balance of these information types.

\vspace{3pt}
\noindent{\bf Influence of iteration number $K$ on the information gap.}
\figref{fig:Loss} illustrates the effect of iteration count $K$ on the information gap.
An increased iteration count $K$ effectively reduces the information gap, enabling the model to capture and retain more features from the original images in the distilled dataset.
However, as $K$ nears 200, further iterations offer minimal gains.
This trend suggests that an appropriate choice of $K$ balances efficient information retention with computational cost across various datasets.

\vspace{3pt}
\noindent{\bf Influence of subset selection.}
Importantly, the selection mechanism is not intrinsic to our framework. To assess alternative strategies, we replaced RDED's selection mechanism with approaches such as Random Selection, K-means Clustering~\cite{forgy1965cluster}, and Herding~\cite{welling2009herding}. A ResNet-18 model was trained on datasets distilled using these various selection strategies, with results summarized in \tabref{tab:selection_ipc10}. Notably, the results indicate that even with Random Selection, our framework achieves competitive performance. Additional results for the case where \IPC = 1 are provided in \darkredtext{Tab. 15} in Appendix.

\begin{figure*}[!t]
    \centering
    \begin{subfigure}{0.32\linewidth}
        \includegraphics[width=\linewidth]{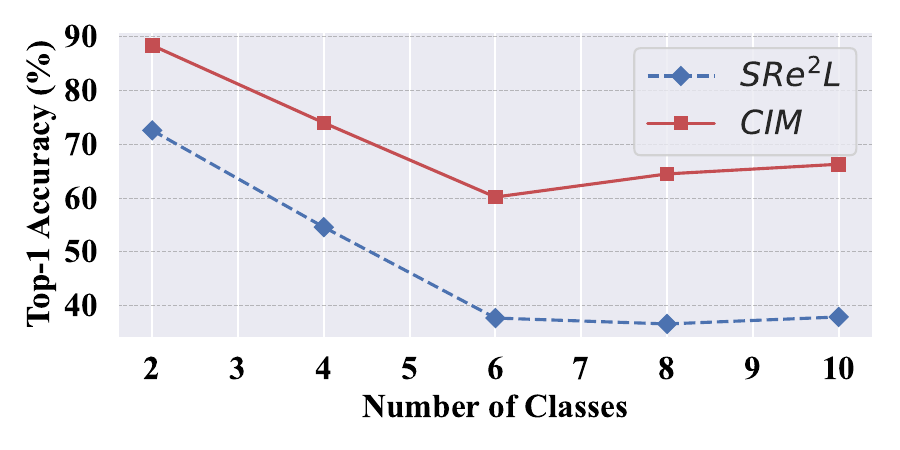}
        \caption{CIFAR-10}
    \end{subfigure}
    \begin{subfigure}{0.32\linewidth}
        \includegraphics[width=\linewidth]{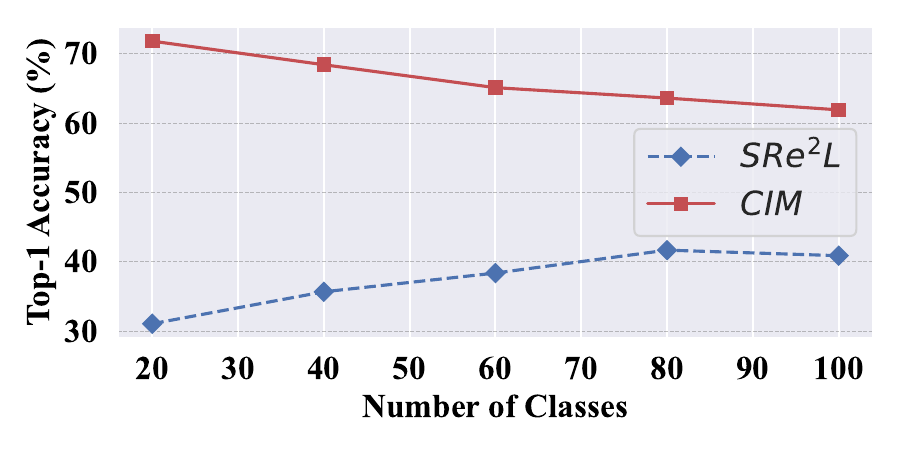}
        \caption{CIFAR-100}
    \end{subfigure}
    \begin{subfigure}{0.32\linewidth}
        \includegraphics[width=\linewidth]{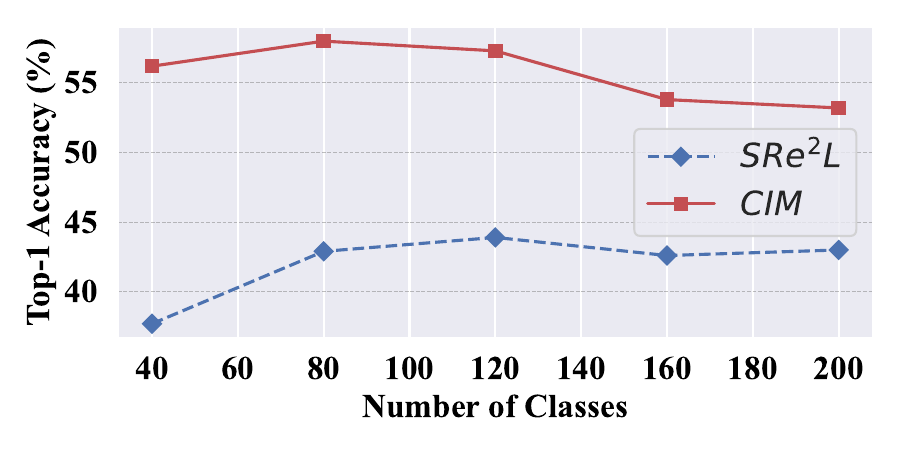}
        \caption{Tiny-ImageNet}
    \end{subfigure}
    \vspace{-10pt}
    \caption{Application of continual learning on various datasets when \IPC=10}
    \label{fig:continual_learning}
\end{figure*}

\begin{table}[t]
    \centering
        \caption{\textbf{Comparison of \algopt using various selection strategies with $\IPC = 10$.} By default, \algopt uses the RDED selection method.}
    \label{tab:selection_ipc10}
    \vspace{-8pt}
     \setlength\tabcolsep{6.5pt}
    \resizebox{.65\textwidth}{!}{
    \begin{tabular}{@{}l|cccc@{}}
        \toprule
                      & Random         & K-means        & Herding        & Ours           \\ \midrule
        CIFAR-10      & 64.5 $\pm$ 0.2 & 65.5 $\pm$ 1.2 & 66.2 $\pm$ 0.8 & \textbf{66.6 $\pm$ 0.1} \\
        CIFAR-100     & 60.4 $\pm$ 0.3 & 60.4 $\pm$ 0.2 & 61.0 $\pm$ 0.1 & \textbf{61.8 $\pm$ 0.1} \\
        Tiny-ImageNet & 51.4 $\pm$ 0.5 & 51.3 $\pm$ 0.2 & 51.3 $\pm$ 0.2 & \textbf{53.2 $\pm$ 0.1} \\
        ImageNet-1k   & 45.3 $\pm$ 0.0 & 44.8 $\pm$ 0.3 & 45.3 $\pm$ 0.1 & \textbf{48.6 $\pm$ 0.2} \\ \bottomrule[1.5pt]
    \end{tabular}}
    \vspace{-13pt}
\end{table}

\vspace{3pt}
\noindent{\bf Application: Continual learning.}
Following prior studies \cite{yin2023squeeze,zhao2023dataset} that leverage synthetic datasets in continual learning to assess the quality of synthetic
data, we employ the GDumb framework \cite{prabhu2020gdumb} for continual learning setup. 
In comparison to SRe$^2$L, our study implements a 5-step class-incremental learning approach using ResNet-18 across CIFAR-10, CIFAR-100, and Tiny-ImageNet datasets, with \IPC = 10.
The results of these experiments are depicted in \figref{fig:continual_learning}.
It is evident that \textit{our results substantially improve upon the baseline methods.}

\section{Conclusion}
In this paper, we demonstrate that the primary limitation of current state-of-the-art dataset distillation methods at various scales is the issue of huge information loss.
To address this, we introduce the \algopt technique, which effectively compresses critical data from images into distilled forms with minimal information loss.
Our extensive experiments reveal that \algopt markedly surpasses existing SOTA dataset distillation techniques across a range of dataset sizes and network architectures.
Additionally, we highlight the efficiency of \algopt by showcasing its ability to distill the ImageNet-1k dataset in just 80 minutes.

\section*{Acknowledgment}
This work was supported in part by the National Science and Technology Major Project (No. 2022ZD0115101), Research Center for Industries of the Future (RCIF) at Westlake University, Westlake Education Foundation, and Westlake University Center for High-performance Computing.

%
%

\bibliographystyle{configuration/splncs04}
\bibliography{resources/main}

\clearpage
\appendix

\section{Limitations}
\label{sec:lim}
Although our \algopt significantly outperforms existing SOTA methods, its primary limitation, as discussed in Section~\ref{sec:comparesota}, is that it cannot surpass the optimization-free paradigm (e.g., RDED~\cite{sun2024diversity}) in terms of efficiency. This constraint limits its applicability in diverse real-world scenarios.

\section{Proof of Proposition~\ref{prop:subop_label}}

\begin{proof}
    \textbf{Define the optimal classification model $f_{opt}$}: Here, we aim to classify the two distributions $\cN_1^s(\mu_1 + s_1, \sigma^2)$ and $\cN_2^s(\mu_2 + s_2, \sigma^2)$ using a decision boundary. Given that both distributions have the same variance, the optimal decision boundary will be linear, and we can employ Bayesian decision theory. The decision boundary is given by the condition:

    \begin{equation}
        P(\cN_1^s|x) = P(\cN_2^s|x)
    \end{equation}

    Applying Bayes' rule, we can transform the above equation into a form of log likelihood ratio:

    \begin{equation}
        \log\left( \frac{P(x | \cN_1^s) P( \cN_1^s )}{P(x | \cN_2^s) P(\cN_2^s)} \right) = 0
    \end{equation}

    Given
    \begin{equation}
        P(x| \cN_i^s) = \frac{1}{\sqrt{2\pi\sigma^2}}e^{-\frac{(x - (\mu_i + s_i))^2}{2\sigma^2}} \, ,
    \end{equation}
    we can further simplify the equation. Note here we assume the prior probabilities $P(\cN_1^s)$ and $P(\cN_2^s)$ are equal. If they are not equal, we would need to consider these prior probabilities. For simplicity, we assume they are equal here. Thus, we can ignore the prior probabilities and focus on the likelihood ratio:
    \begin{equation}
        \frac{(x - (\mu_1 + s_1))^2}{2\sigma^2} = \frac{(x - (\mu_2 + s_2))^2}{2\sigma^2}
    \end{equation}
    Expanding and simplifying the above equation, we can find the value of $x$, which will be the decision boundary for classification:
    \begin{align}
        x(\mu_1 + s_1 - \mu_2 - s_2) & = \frac{(\mu_1 + s_1)^2 - (\mu_2 + s_2)^2}{2}                            \\
        x                            & = \frac{(\mu_1 + s_1)^2 - (\mu_2 + s_2)^2}{2(\mu_1 + s_1 - \mu_2 - s_2)} \\
        x                            & = \frac{(\mu_1 + s_1 + \mu_2 + s_2)}{2}
    \end{align}
    This value of $x$ defines the position of the linear decision boundary $f_{opt}$. This result provides an explicit expression for the optimal linear decision boundary, allowing us to classify based on the means and variances of the two distributions.
\end{proof}

\section{Selecting Key Subsets from Original Dataset}
\label{sec:keysamples}
Our goal is to devise an effective method for identifying the crucial samples that are instrumental in benefiting the \Relabel process.
Motivated by the Summary in Section~\ref{sec:motivation} and Proposition~\ref{prop:subop_label}---which suggests that each chosen sample pair $(\xx, y)$ is supposed to receive an informative and accurate label $\phi_{\mtheta_\cT}(\xx)$ from the pre-trained model during the \Relabel phase---the goal then relaxes to find the subsets of $\cT$ that include samples that are relabeled most accurately by the pre-trained model $\phi_{\mtheta_\cT}$.
Thus, we introduce a loss-based importance score $s$ for each sample pair $(\xx, y)$, defined as $s = -\ell(\phi_{\mtheta_\cT}(\xx), y)$.
The key sample selection procedure is elaborated below.

Let $\cT_c := \{ (\xx, y) \mid (\xx, y) \in \cT, y = c \}$ represents the subset of the dataset $\cT$, containing only those samples $(\xx, y)$ that are labeled with the class $c$.
For each class data $\cT_c$, we identify the key samples $\xx$ based on their importance scores $s$, forming $\IPC$ subsets of key samples as
\begin{equation}
    \begin{aligned}
        \cS_c = \left\{ \{ \xx_{(j,i)}, y_{(j,i)} \}_{j=1}^N \right\}_{i=1}^{\IPC} \\
        \textup{s.t.} \;\; s_{(j,i)} = -\textup{CE}(\phi_{\mtheta_\cT}(\xx_{(j,i)}), y_{(j,i)}) \geq \bar{s} \,,
    \end{aligned}
\end{equation}
where $\bar{s}$ denotes a predetermined threshold\footnote{
    This threshold is used for selecting $\IPC \times N$ samples for further processing.
}, $N$ indicates the number of samples within each selected subset, and CE denotes the \texttt{CrossEntropyLoss}.
We further simplify the whole procedure due to the computational overhead and diversity issue, as explained below
\begin{enumerate}[nosep, leftmargin=12pt]
    \item computing the scores $s$ for all samples $\xx$ in data $\cT_c$ presents a significant computational challenge;
    \item focusing solely on samples that closely align with the true label can lead to a lack of diversity.
\end{enumerate}
Therefore, we utilize a pre-selection strategy inspired by \cite{sun2024diversity}, which involves selecting a subset\footnote{
    We use the default size of this uniform-randomly selected subset in \cite{sun2024diversity}. see details in Section~\ref{sec:exp}.
} $\cT_c^\prime \subset \cT_c$ uniformly at random to serve as a proxy for the entire $\cT_c$.
Such a pre-selection strategy not only promotes diversity in the data but also lessens the computational load~\cite{sun2024diversity}, thereby laying the groundwork for our subsequent score-based sample selection process.

\newpage
\section{Proof of Theorem~\ref{thm:bound_ei_gap}}

We initiate our derivation by examining the Kullback-Leibler divergence between the effective information distributions of $x _i$ and $\hat{x} _i$ as observed by the group $\mathcal{R}$:

\begin{align}
    &\text{D} _\text{KL}(\cP _{x _i | \mathcal{R}} || \cP _{\hat{x} _i | \mathcal{R}}) = \mathbb{E} _z  \left[ \log \frac{\cP _{x _i | \mathcal{R}}(z)}{\cP _{\hat{x} _i | \mathcal{R}}(z)}\right] 
    &= \mathbb{E} _z  \left[\log \cP _{x _i | \mathcal{R}}(z) - \log \cP _{\hat{x} _i | \mathcal{R}}(z)\right] \\ \nonumber
    &\leq \mathbb{E} _z \left[| \log \cP _{x _i | \mathcal{R}}(z) - \log \cP _{\hat{x} _i | \mathcal{R}}(z) |\right] \, .
\end{align}

By applying kernel density estimation to $p(\cdot)$, we obtain:

\begin{align}
    &\mathbb{E} _z \left[| \log \cP _{x _i | \mathcal{R}}(z) - \log \cP _{\hat{x} _i | \mathcal{R}}(z) | \right]\\
    &=\mathbb{E} _{\xi _j} \left[| \log \sum _{\xi _k}\left[\frac{1}{(2\pi)^{\frac{1}{d}}} \exp(\frac{- \| \xi _j(x _i) - \xi _k(x _i) \|^2}{2})\right] \right. \\
    & \qquad  \qquad  \left. - \log \sum _{\xi _k}{\left[\frac{1}{(2\pi)^{\frac{1}{d}}} \exp(\frac{- \| \xi _j(\hat{x} _i) - \xi _k(\hat{x} _i) \|^2}{2})\right]} |\right]\, ,
\end{align}

where $d$ denotes the dimension of $x$ and a primary term exists within $\log(\cdot)$. Therefore, we approximate:

\begin{align}
    &\approx \mathbb{E}_{\xi_j} \left[ \left| \sum_{\xi_k} \left[ \log\left(\frac{1}{(2\pi)^{\frac{1}{d}}} 
    \exp\left(\frac{- \| \xi_j(x_i) - \xi_k(x_i) \|^2}{2}\right)\right)\right] \right. \right. \\
    & \quad - \left. \left. \sum_{\xi_k} \left[ \log\left(\frac{1}{(2\pi)^{\frac{1}{d}}} 
    \exp\left(\frac{- \| \xi_j(\hat{x}_i) - \xi_k(\hat{x}_i) \|^2}{2}\right)\right) \right] \right| \right] \\
    &= \mathbb{E}_{\xi_j} \left[ \left| \sum_{\xi_k} \left[ \frac{- \| \xi_j(x_i) - \xi_k(x_i) \|^2}{2} \right] - \sum_{\xi_k} \left[ \frac{- \| \xi_j(\hat{x}_i) - \xi_k(\hat{x}_i) \|^2}{2} \right] \right| \right] \\
    &= \mathbb{E}_{\xi_j} \left[ \left| \sum_{\xi_k} \left[ \frac{- \| \xi_j(x_i) - \xi_k(x_i) \|^2}{2} \right] - \sum_{\xi_k} \left[ \frac{- \| \xi_j(\hat{x}_i) - \xi_k(\hat{x}_i) \|^2}{2} \right] \right| \right] \\
    &\leq \mathbb{E}_{\xi_j} \left[ \sum_{\xi_k} \left[ \left| \frac{- \| \xi_j(x_i) - \xi_k(x_i) \|^2}{2} - \frac{- \| \xi_j(\hat{x}_i) - \xi_k(\hat{x}_i) \|^2}{2} \right| \right] \right] \\
    &\leq \mathbb{E}_{\xi_j} \left[ \sum_{\xi_k} \left[ \left| \| \xi_j(\hat{x}_i) - \xi_k(\hat{x}_i) \|^2  - \| \xi_j(x_i) - \xi_k(x_i) \|^2 \right| \right] \right] \\
    &= \mathbb{E}_{\xi_j} \left[ \sum_{\xi_k \neq \xi_j} \left[ \left| \| \xi_j(\hat{x}_i) - \xi_k(\hat{x}_i) \|^2  - \| \xi_j(x_i) - \xi_k(x_i) \|^2 \right| \right] \right] \\
    &= \mathbb{E}_{\xi_j} \Bigg[ \sum_{\xi_k \neq \xi_j} \Bigg( \left| \| \xi_j(\hat{x}_i) \|^2 
    - 2\langle \xi_j(\hat{x}_i), \xi_k(\hat{x}_i) \rangle \right. + \| \xi_k(\hat{x}_i) \|^2 - \| \xi_j(x_i) \|^2 \\
    & \quad + 2\langle \xi_j(x_i), \xi_k(x_i) \rangle - \| \xi_k(x_i) \|^2 \Bigg| \Bigg) \Bigg] \\
    &\leq (|\mathcal{R}|-1) \cdot \mathbb{E}_{\xi_j} \left[ \left| \| \xi_j(\hat{x}_i) \|^2 - \| \xi_j(x_i) \|^2 \right| \right] + \sum_{\xi_k \neq \xi_j} \left[ \left| \| \xi_k(\hat{x}_i) \|^2 - \| \xi_k(x_i) \|^2 \right| \right] \\
    & \quad + \mathbb{E}_{\xi_j} \left[ \sum_{\xi_k \neq \xi_j} \left[ \left| 2\langle \xi_j(x_i), \xi_k(x_i) \rangle - 2\langle \xi_j(\hat{x}_i), \xi_k(\hat{x}_i) \rangle \right| \right] \right] \\
    &\approx (|\mathcal{R}|-1) \cdot \mathbb{E}_{\xi_j} \left[ \left| \| \xi_j(\hat{x}_i) \|^2 - \| \xi_j(x_i) \|^2 \right| \right] + \sum_{\xi_k \neq \xi_j} \left[ \left| \| \xi_k(\hat{x}_i) \|^2 - \| \xi_k(x_i) \|^2 \right| \right] + o(1) \\
    &\approx (|\mathcal{R}|-1) \cdot \mathbb{E}_{\xi_j} \left[ \left| \| \xi_j(\hat{x}_i) \|^2 - \| \xi_j(x_i) \|^2 \right| \right] 
    & \quad + \lambda
\end{align}

To minimize this term, we can use an alternative one:

\begin{equation}
    \mathbb{E} _{\xi _j} \left[ \| \xi _j(\hat{x} _i) - \xi _j(x _i)\|^2  \right] \, .
\end{equation}

\section{Detailed Analysis for Existing Information Extraction-Based Approaches}
\label{sec:analy_existing_works}

\paragraph{Results without relabel.}

We evaluate the data with and without employing relabeling. The results are presented in Table~\ref{tab:lic_no_relabel}.
The visualizations of the images synthesized by state-of-the-art methods are presented in Section~\ref{sec:appendix_visualization}.

\begin{table*}[]
    \centering
        \caption{Testing was conducted with/without the application of relabeling technique. The experiment utilized the ResNet-18 model, with an \IPC value set to 10.}
    \label{tab:lic_no_relabel}
    \resizebox{1.\textwidth}{!}{
        \begin{tabular}{@{}c|cccc|cccc@{}}
            \toprule
            Relabel       & \multicolumn{4}{c|}{Without Relabel} & \multicolumn{4}{c}{With Relabel}                                                                                                                \\ \midrule
            Dataset       & G-VBSM                            & SRe$^2$L             & RDED           & Ours           & G-VBSM         & SRe$^2$L & RDED           & Ours           \\ \midrule
            CIFAR-10      & 32.3 $\pm$ 0.7                    & 10.9 $\pm$ 0.5                      & 35.2 $\pm$ 0.6 & 51.6 $\pm$ 0.6 & 36.3 $\pm$ 0.7 & 39.4 $\pm$ 0.9          & 47.3 $\pm$ 0.5 & 66.2 $\pm$ 0.9 \\
            CIFAR-100     & 10.2 $\pm$ 0.2                    & 1.2 $\pm$ 0.2                       & 21.5 $\pm$ 0.3 & 40.7 $\pm$ 0.3 & 47.0 $\pm$ 0.4 & 42.7 $\pm$ 0.5          & 53.4 $\pm$ 0.3 & 61.7 $\pm$ 0.2 \\
            Tiny-ImageNet & 0.6 $\pm$ 0.1                     & 0.6 $\pm$ 0.0                       & 14.7 $\pm$ 0.4 & 27.0 $\pm$ 0.5 & 37.7 $\pm$ 0.3 & 43.6 $\pm$ 0.5          & 48.4 $\pm$ 0.3 & 53.3 $\pm$ 0.1 \\
            ImageNet-1k   & 0.8 $\pm$ 0.1                     & 1.1 $\pm$ 0.0                       & 19.7 $\pm$ 0.3 & 22.0 $\pm$ 1.8 & 35.7 $\pm$ 0.2 & 31.1 $\pm$ 0.1          & 41.1 $\pm$ 0.2 & 48.7 $\pm$ 0.2 \\ \bottomrule[1.5pt]
        \end{tabular}
    }
\end{table*}

\section{Detailed Framework of Our \algopt}
The structure of our \algopt framework is outlined in Algorithm~\ref{alg:framework}.


\begin{algorithm}[!t]
    \caption{\small
        An efficient framework for dataset distillation.
    }\label{alg:framework}
    \begin{algorithmic}
        \STATE {\bfseries Input:} Original full dataset $\cT$, a corresponding pre-trained observer model $\phi_{\theta_{\cT}}$ and initial $\cS = \emptyset$.
        \STATE {\bfseries Parameters:} The number $N$ for squeezing images, the number $M$ of compression iterations, the size of $\cT_c^\prime$.
        \FOR{$\cT_c^\prime \subset \cT_c \subset \cT$}
        \STATE \COMMENT{\textcolor{red}{\textbf{Stage 1. } Selecting Subsets of Key Samples}}
        \FOR{$(\xx_i, y_i) \in \cT_c^\prime$}
        \STATE Calculate $s_i = -\ell(\phi_{\mtheta_\cT}(\xx_i), y_i)$
        \ENDFOR
        \STATE Select top-($N \times \IPC$) images $\{ \xx_i \}_{i=1}^{N \times \IPC}$ via $s_i$
        \STATE \COMMENT{\textcolor{red}{\textbf{Stage 2. } Compressing Effective Information}}
        \FOR{$j=1$ {\bfseries to} \IPC}
        \STATE Initialize distilled images as $\widetilde{\xx}_j$
        \FOR{$m=1$ {\bfseries to} $M$}
        \STATE $\Delta \xx \leftarrow \Delta \xx - \nabla_{\Delta \xx} \cL_{\Delta \xx}$
        \ENDFOR
        \STATE \COMMENT{\textcolor{red}{\textbf{Stage 3. } \Relabel}}
        \STATE Relabel $\widetilde{\xx}_j \gets \widetilde{\xx}_j + \Delta \xx^\star$ with $\widetilde{y}_j$
        \STATE $\cS=\cS \cup \{(\widetilde{\xx}_j, \widetilde{y}_j)\}$
        \ENDFOR
        \ENDFOR
        \STATE {\bfseries Output:} Small distilled dataset $\cS$
    \end{algorithmic}
\end{algorithm}

\section{Experiment Details}
\label{sec:appendix_details}

\paragraph{Datasets.}
In addition to the datasets described in Section~\ref{sec:expset}, we note that prevalent dataset distillation techniques struggle to scale to large, high-resolution datasets.
In Table~\ref{tab:datasets}, we present some information about the dataset, including the number of classes, the number of images per class in the training set, and the test set.
\begin{table*}[]
    \centering
        \caption{Details about the datasets}
    \label{tab:datasets}
    \begin{tabular}{@{}c|ccc@{}}
        \toprule
        Dataset       & Num of Classes & IPC of Training Set & IPC of Test Set \\ \midrule
        CIFAR-10      & 10             & 5000            & 1000           \\
        CIFAR-100     & 100            & 500             & 100            \\
        Tiny-ImageNet & 200            & 500             & 50             \\
        ImageNet-1k   & 1000           & 732 - 1300      & 50             \\ \bottomrule[1.5pt]
    \end{tabular}
\end{table*}

\paragraph{Models.}

The experiment employed a multitude of pre-trained models, and we delineated their accuracies in Table~\ref{tab:pretrained_accuracy}. These results are furnished solely for reference purposes.
\begin{table}[]
    \centering
        \caption{Accuracy of pre-trained models.}
    \label{tab:pretrained_accuracy}
    \begin{tabular}{@{}c|ccc@{}}
        \toprule
        Dataset                        & Model              & Size             & Accuracy \\ \midrule
        \multirow{2}{*}{CIFAR-10}      & ResNet-18 & 32 $\times$ 32   & 93.86    \\
                                       & Conv-3          & 32 $\times$ 32   & 82.24    \\ \midrule
        \multirow{2}{*}{CIFAR-100}     & ResNet-18 & 32 $\times$ 32   & 72.27    \\
                                       & Conv-3          & 32 $\times$ 32   & 61.27    \\ \midrule
        \multirow{2}{*}{Tiny-ImageNet} & ResNet-18 & 64 $\times$ 64   & 61.98    \\
                                       & Conv-4          & 64 $\times$ 64   & 49.73    \\ \midrule
        \multirow{2}{*}{ImageNet-1k}   & ResNet-18           & 224 $\times$ 224 & 69.31    \\
                                       & Conv-4          & 64 $\times$ 64   & 43.6     \\ \bottomrule[1.5pt]
    \end{tabular}
\end{table}

\paragraph{Baselines.}

We benchmark our proposed \algopt against a range of SOTA distillation techniques capable of handling large, high-resolution datasets.
\begin{itemize}[nosep, leftmargin=12pt]
    \item G-VBSM \cite{shao2023generalized} surpasses one-sided methods like SRe$^2$L by creating synthetic datasets with richer information and better generalization across \emph{various backbones, layers, and statistics}.
    \item SRe$^2$L \cite{yin2023squeeze} is a novel entrant that \emph{efficiently handles ImageNet-1K}, significantly outpacing other methods in managing large, high-resolution datasets and serving as our primary comparison point.
    \item RDED \cite{sun2024diversity} enables the compression of large-scale, high-resolution datasets while maintaining diversity and realism, \emph{significantly reducing the time required} for training neural networks like ResNet-18 on ImageNet-1K.
    \item DataDAM \cite{sajedi2023datadam} \emph{efficiently distills images across multiple resolutions and scales} by matching spatial attention maps between real and distilled samples at various layers within families of randomly initialized neural networks.
    \item ADD \cite{zhang2023accelerating} demonstrates \emph{effective scalability across varying dataset resolutions}, enhancing distillation speed through model augmentation.
    \item IDM \cite{zhao2023improved} presents an efficient dataset condensation technique utilizing distribution matching, offering a scalable alternative to computationally demanding optimization-focused methods \cite{zhao2020dataset,cazenavette2022dataset}.
    \item CDA \cite{yin2023dataset} achieves superior accuracy on large-scale datasets like ImageNet-1K and 21K and significantly narrows the performance gap compared to full-data training counterparts.
    \item WMDD \cite{liu2024dataset} presents a novel dataset distillation approach that \emph{employs the Wasserstein distance to enhance distribution matching}, achieving state-of-the-art performance by effectively capturing the essential representations of extensive datasets in synthetic forms.
    \item DATM \cite{guo2023towards} stands out by occasionally \emph{surpassing the training performance of the full original dataset}, for instance, achieving $\IPC=100$ on CIFAR-100.
    \item DREAM \cite{liu2023dream} introduces an \emph{efficient technique while also delivering the most remarkable results}.
    \item FreD \cite{shin2023frequency} is a novel parameterization method for dataset distillation that \emph{operates in the frequency domain}, significantly reducing the budget for synthesizing a small-sized synthetic dataset while preserving the original dataset's information and consistently improving the performance of existing distillation methods.
\end{itemize}

\paragraph{Evaluating main results.}

For both dataset distillation and performance evaluation, we employ identical neural network architectures.
Consistent with previous studies \cite{cazenavette2022dataset, cui2023scaling, zhao2023improved}, we use Conv-3 for CIFAR-10 and CIFAR-100 distillation tasks and Conv-4 for Tiny-ImageNet (with the exception of DREAM, which utilizes Conv-3) and ImageNet-1K distillation. In line with \cite{cazenavette2022dataset, cui2023scaling}, MTT and TESLA apply a reduced resolution for distilling $224 \times 224$ images. According to \cite{yin2023squeeze}, for retrieving and evaluating distilled datasets, SRe$^2$L and \algopt adopt ResNet-18.

\paragraph{Evaluating the distilled dataset.}

We detail the hyperparameter configurations for distilling datasets in Table~\ref{tab:settings_distill}.
Consistent with recent works~\cite{yin2023dataset,yin2023squeeze,shao2023generalized}, the evaluation on the distilled dataset follows the parameters outlined in Table~\ref{tab:settings_eval}.
Furthermore, we implement Differentiable Siamese Augmentation (DSA) as described by \cite{zhao2021DSA} to enhance images during both the distillation and evaluation phases of our experiments.

\begin{table*}[t]
    \centering
        \caption{Hyperparameter setting.}
    \label{tab:settings}
    \begin{subtable}[b]{1\textwidth}
        \centering
        \caption{Data Synthesis.}
        \begin{tabular}{@{}l|ll@{}}
            \toprule
            Config         & Value    & Explanation                                                                                     \\ \midrule
            Iteration      & 200      & NA                                                                                              \\
            Optimizer      & AdamW    & $\beta_{1}$, $\beta_{2}$=(0.9, 0.999)                                                           \\
            Learning Rate  & 0.01     & NA                                                                                              \\
            Initialization & RDED     & Initialized using images from training dataset                                                  \\
            Factor         & 2        & NA                                                                                              \\
            Mipc           & 300      & NA                                                                                              \\
            Depth          & Deep/Mid & \begin{tabular}[c]{@{}l@{}}Deep for ConvNet and ResNet(modified),\\ Mid for ResNet\end{tabular} \\ \bottomrule[1.5pt]
        \end{tabular}
        \label{tab:settings_distill}
    \end{subtable}
    \begin{subtable}[b]{1\textwidth}
        \centering
        \caption{Evaluation.}
        \begin{tabular}{@{}l|ll@{}}
            \toprule
            Config        & Value         & Explanation                                                                                                                                                                                                                             \\ \midrule
            Epochs        & 300/1000      & \begin{tabular}[c]{@{}l@{}}300 for ImageNet-1k,\\ 1000 for default\end{tabular}                                                                                                                                                         \\
            Optimizer     & AdamW         & NA                                                                                                                                                                                                                                      \\
            Learning Rate & 0.001         & NA                                                                                                                                                                                                                                      \\
            Batch Size    & 10/50/100/200 & \begin{tabular}[c]{@{}l@{}}10 for 0 $\textless$ Num of Images $\leq$ 10,\\ 50 for 10 $\textless$ Num of Images $\leq$ 500,\\ 100 for 500 $\textless$ Num of Images $\leq$ 20000,\\ 200 for 20000 $\textless$ Num of Images\end{tabular} \\
            Scheduler     & MultiStepLR   & \begin{tabular}[c]{@{}l@{}}milestones=[2 $\times$ epochs // 3, 5 $\times$ epochs // 6]\\ gamma=0.2\end{tabular}                                                                                                                         \\
            Augmentation  & DSA strategy  & color, crop, cutout, flip, scale, rotate                                                                                                                                                                                                \\ \bottomrule[2pt]
        \end{tabular}
        \label{tab:settings_eval}
    \end{subtable}
\end{table*}

\paragraph{Differentiable Siamese Augmentation (DSA).}
We employ DSA (Differentiable Siamese Augmentation) as a method for image augmentation. To enhance clarity, we outline the DSA operations utilized in Table~\ref{tab:dsa_strategy}, along with their corresponding transformations and probabilities.


\begin{table}[h]
    \centering
    \caption{Differentiable Siamese Augmentation(DSA) and ratios}
        \label{tab:dsa_strategy}
    \begin{tabular}{@{}l|ll@{}}
        \toprule[2pt]
        DSA    & Transform              & Ratio                                                                                  \\ \midrule
        Color  & Color Jitter           & \begin{tabular}[c]{@{}l@{}}Brightness=1.0\\ Saturation=2.0\\ Contrast=0.5\end{tabular} \\
        Crop   & Random Crop            & Crop Pad=0.125                                                                         \\
        Cutout & Random Cutout          & Cutout=0.5                                                                             \\
        Flip   & Random Horizontal Flip & Flip=0.5                                                                               \\
        Scale  & Random Scale           & Scale=1.2                                                                              \\
        Rotate & Random Rotation        & Rotate=15.0                                                                            \\ \bottomrule[2pt]
    \end{tabular}
\end{table}

\begin{figure}[h]
    \centering
    \includegraphics[width=0.5\textwidth]{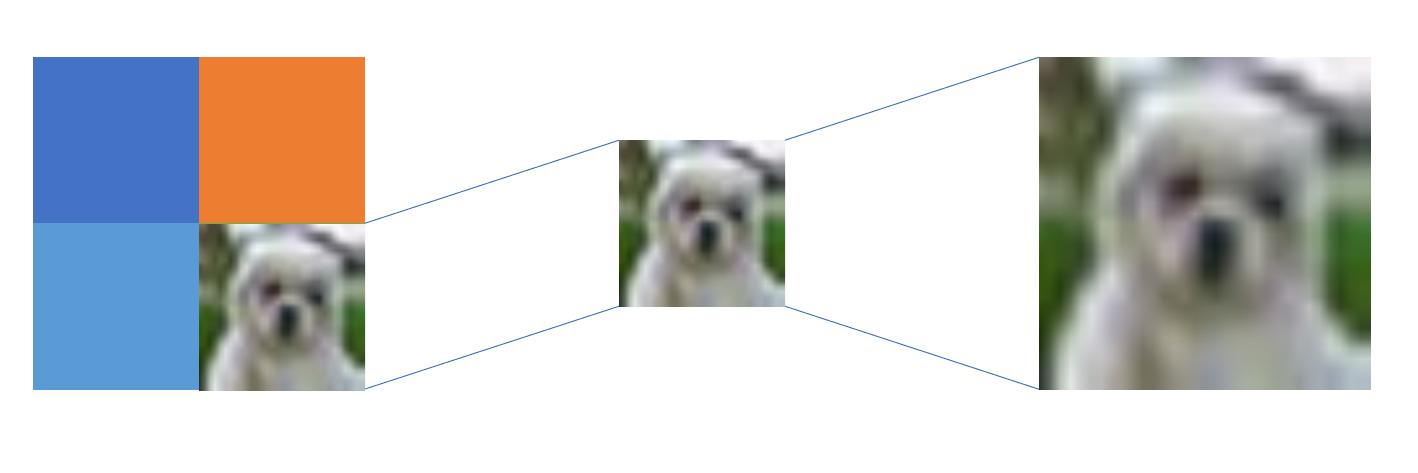}
    \caption{Visualization of factor crop.}
    \label{fig:factor_crop}
\end{figure}

\paragraph{Factor Crop.}
We have integrated a cropping method termed 'factor crop', which is applied prior to DSA. This technique enhances sample diversity by precisely extracting regions from specific areas of the images and resizing them to their original dimensions. As an augmentation method, it bolsters the model's generalizability, as depicted in Figure~\ref{fig:factor_crop}. This approach functions as a substitute for \texttt{RandomCrop}, thereby preserving the semantic integrity of compressed images within the distilled dataset.




\section{Experiment Results}
\label{sec:appendix_results}

\paragraph{Comparison with more datasets and baselines.}

In addition to the experiments discussed in Section~\ref{sec:comparesota}, we further benchmark our proposed \algopt against a broader set of baselines, encompassing recent contributions~\cite{yu2023dataset,shao2023generalized,sun2023diversity,liu2023dataset}.
The outcomes, presented in Table~\ref{tab:results_subset_imagenet} and Table~\ref{tab:results_reproduced}, consistently affirm the superior performance of \algopt in dataset distillation tasks.

\begin{table*}[h]
    \centering
    \caption{\small
    Comparison with more SOTA dataset distillation methods.
    }
    \setlength\tabcolsep{5.2pt}
    \resizebox{1\textwidth}{!}{
        \begin{tabular}{@{}cc|ccccc|ccccc@{}}
            \toprule[2pt]
            \multicolumn{2}{c|}{Architecture} & \multicolumn{5}{c|}{ConvNet}                                                                     & \multicolumn{5}{c}{ResNet-18}                                                                    \\ \midrule
            Dataset              & IPC        & DataDAM        & ADD  & IDM            & RDED                    & \algopt (Ours) & G-VBSM         & CDA  & WMDD                    & RDED           & \algopt (Ours) \\ \midrule
                                & 1          & 32.0 $\pm$ 1.2 & 49.2 & 45.6 $\pm$ 0.7 & 23.5 $\pm$ 0.3          & \textbf{37.4 $\pm$ 0.4}       & -              & -    & -                       & 22.9 $\pm$ 0.4 & \textbf{32.3 $\pm$ 0.2}       \\
            CIFAR-10             & 10         & 54.2 $\pm$ 0.8 & 67.1 & 58.6 $\pm$ 0.1 & 50.2 $\pm$ 0.3          & \textbf{61.4 $\pm$ 0.4}       & 53.5 $\pm$ 0.6 & -    & -                       & 37.1 $\pm$ 0.3 & \textbf{66.2 $\pm$ 0.9}       \\
                                & 50         & 67.0 $\pm$ 0.4 & 73.8 & 67.5 $\pm$ 0.1 & 68.4 $\pm$ 0.1          & \textbf{74.7 $\pm$ 0.2}       & 59.2 $\pm$ 0.4 & -    & -                       & 62.1 $\pm$ 0.1 & \textbf{85.1 $\pm$ 0.3}       \\ \midrule
                                & 1          & 14.5 $\pm$ 0.5 & 29.8 & 20.1 $\pm$ 0.3 & 19.6 $\pm$ 0.3          & \textbf{27.4 $\pm$ 0.4}       & 25.9 $\pm$ 0.5 & -    & -                       & 11.0 $\pm$ 0.3 & \textbf{31.1 $\pm$ 0.7}       \\
            CIFAR-100            & 10         & 34.8 $\pm$ 0.5 & 45.6 & 45.1 $\pm$ 0.1 & 48.1 $\pm$ 0.3          & \textbf{49.3 $\pm$ 0.3}       & 59.5 $\pm$ 0.4 & -    & -                       & 42.6 $\pm$ 0.2 & \textbf{61.7 $\pm$ 0.2}       \\
                                & 50         & 49.4 $\pm$ 0.3 & 52.6 & 50.0 $\pm$ 0.2 & \textbf{57.0 $\pm$ 0.1} & 55.1 $\pm$ 0.2                & 65.0 $\pm$ 0.5 & -    & -                       & 62.6 $\pm$ 0.1 & \textbf{67.0 $\pm$ 0.1}       \\ \midrule
                                & 1          & 8.3 $\pm$ 0.4  & -    & 10.1 $\pm$ 0.2 & 12.0 $\pm$ 0.1          & \textbf{25.4 $\pm$ 0.2}       & -              & -    & 7.6 $\pm$ 0.2           & 9.7 $\pm$ 0.4  & \textbf{25.1 $\pm$ 0.4}       \\
            Tiny ImageNet        & 10         & 18.7 $\pm$ 0.3 & -    & 21.9 $\pm$ 0.2 & 39.6 $\pm$ 0.1          & \textbf{42.2 $\pm$ 0.2}       & -              & -    & 41.8 $\pm$ 0.1          & 41.9 $\pm$ 0.2 & \textbf{53.3 $\pm$ 0.1}       \\
                                & 50         & 28.7 $\pm$ 0.3 & -    & 27.7 $\pm$ 0.3 & \textbf{47.6 $\pm$ 0.2} & 47.0 $\pm$ 0.2                & -              & 48.7 & \textbf{59.4 $\pm$ 0.5} & 58.2 $\pm$ 0.1 & 58.6 $\pm$ 0.1                \\ \midrule
                                & 1          & 2.0 $\pm$ 0.1  & -    & -              & -                       & \textbf{5.8 $\pm$ 0.1}        & -              & -    & 3.2 $\pm$ 0.3           & 6.6 $\pm$ 0.2  & \textbf{7.3 $\pm$ 0.3}        \\
            ImageNet-1k          & 10         & 6.3 $\pm$ 0.0  & -    & -              & -                       & \textbf{24.5 $\pm$ 0.1}       & 31.4 $\pm$ 0.5 & -    & 38.2 $\pm$ 0.2          & 42.0 $\pm$ 0.1 & \textbf{48.7 $\pm$ 0.2}       \\
                                & 50         & 15.5 $\pm$ 0.2 & -    & -              & -                       & \textbf{37.8 $\pm$ 0.1}       & 51.8 $\pm$ 0.4 & 53.5 & 57.6 $\pm$ 0.5          & 56.5 $\pm$ 0.1 & \textbf{60.4 $\pm$ 0.1}       \\ \bottomrule[2pt]
        \end{tabular}
    }
    \label{tab:results_subset_imagenet}
\end{table*}

\begin{table*}[]
    \centering
        \caption{Comparison with other SOTA dataset distillation methods that we reproduced.}
    \label{tab:results_reproduced}
    \resizebox{0.75\textwidth}{!}{
        \begin{tabular}{@{}cccccc@{}}
            \toprule
            \multicolumn{2}{c}{Architecture} & \multicolumn{4}{c}{ConvNet}                                                                                    \\ \midrule
            Dataset                          & \multicolumn{1}{c|}{IPC}    & DATM           & DREAM          & FreD           & \algopt (Ours) \\ \midrule
                                             & \multicolumn{1}{c|}{1}      &                & 40.8 $\pm$ 1.2 & 33.2 $\pm$ 1.2 & \textbf{50.0 $\pm$ 0.3}       \\
            CIFAR-10                         & \multicolumn{1}{c|}{10}     & 60.1 $\pm$ 0.3 & 64.2 $\pm$ 0.1 & 46.6 $\pm$ 0.6 & \textbf{72.2 $\pm$ 0.2}       \\
                                             & \multicolumn{1}{c|}{50}     & 63.1 $\pm$ 1.0 & 72.2 $\pm$ 0.1 & 46.3 $\pm$ 0.4 & \textbf{78.3 $\pm$ 0.1}       \\ \midrule
                                             & \multicolumn{1}{c|}{1}      & -              & 22.4 $\pm$ 0.4 & 15.5 $\pm$ 0.2 & \textbf{42.7 $\pm$ 0.3}       \\
            CIFAR-100                        & \multicolumn{1}{c|}{10}     & 27.7 $\pm$ 0.3 & 41.3 $\pm$ 0.6 & -              & \textbf{54.8 $\pm$ 0.2}       \\
                                             & \multicolumn{1}{c|}{50}     & 42.8 $\pm$ 0.2 & 48.3 $\pm$ 0.2 & -              & \textbf{56.6 $\pm$ 0.1}       \\ \midrule
                                             & \multicolumn{1}{c|}{1}      & -              & -              & 5.8 $\pm$ 0.2  & \textbf{31.7 $\pm$ 0.4}       \\
            Tiny ImageNet                    & \multicolumn{1}{c|}{10}     & 28.8 $\pm$ 0.2 & -              & -              & \textbf{46.3 $\pm$ 0.2}       \\
                                             & \multicolumn{1}{c|}{50}     & 36.1 $\pm$ 0.1 & -              & -              & \textbf{47.4 $\pm$ 0.1}       \\ \midrule
                                             & \multicolumn{1}{c|}{1}      & -              & -              & -              & \textbf{14.5 $\pm$ 0.3}       \\
            ImageNet-1k                      & \multicolumn{1}{c|}{10}     & -              & -              & -              & \textbf{24.0 $\pm$ 0.5}       \\
                                             & \multicolumn{1}{c|}{50}                          & -              & -              & -              & \textbf{37.3 $\pm$ 0.1}       \\ \bottomrule[1.5pt]
        \end{tabular}
    }
\end{table*}

\subsection{Efficiency Comparison}


Beyond assessing performance in Section~\ref{sec:efficiencycomp}, we expand our evaluation to include additional baselines.
Results presented in Table~\ref{tab:efficiency} underscore the exceptional efficiency of our proposed \algopt, which also requires the least GPU memory.

\begin{table*}[t]
    \centering
        \caption{\small
    Efficiency comparison with SOTA methods with Conv-4 on Tiny-ImageNet.
    }
    \label{tab:efficiency}
    \begin{tabular}{@{}c|ccc|ccc@{}}
        \toprule
        \multirow{2}{*}{Architecture} & \multicolumn{3}{c|}{Time Cost (s)} & \multicolumn{3}{c}{Peak Memory (GB)}                                \\ \cmidrule(l){2-7}
                                      & DREAM                              & DATM                                 & Ours  & DREAM & DATM  & Ours \\ \midrule
        Conv-4                        & 33906.17                           & 12470.90                             & 13.02 & 15.92 & 20.16 & 0.65 \\ \bottomrule[1.5pt]
    \end{tabular}
\end{table*}

\section{Ablation}
\label{sec:appendix_ablation}

\paragraph{Selection.}

To compare our selection strategy with others, namely Random, K-means~\cite{forgy1965cluster}, and Herding~\cite{welling2009herding}, we trained ResNet-18 using datasets distilled through various strategies, as delineated in Tables~\ref{tab:selection_ipc1}.

\begin{table*}[t]
    \centering
        \caption{Comparision with other selection strategies, with $\IPC = 1$}
    \label{tab:selection_ipc1}
    \begin{tabular}{@{}c|cccc@{}}
        \toprule
                      & Random         & K-means        & Herding        & Ours           \\ \midrule
        CIFAR-10      & 28.3 $\pm$ 0.4 & 30.3 $\pm$ 0.4 & 31.0 $\pm$ 0.3 & 31.3 $\pm$ 0.5 \\
        CIFAR-100     & 25.1 $\pm$ 0.0 & 25.5 $\pm$ 0.3 & 26.9 $\pm$ 0.2 & 30.6 $\pm$ 0.1 \\
        Tiny-ImageNet & 19.5 $\pm$ 0.5 & 19.6 $\pm$ 1.0 & 19.5 $\pm$ 0.2 & 25.5 $\pm$ 0.0 \\
        ImageNet-1k   & 5.0 $\pm$ 0.0  & 5.6 $\pm$ 0.1  & 5.6 $\pm$ 0.1  & 7.1 $\pm$ 0.2  \\ \bottomrule[1.5pt]
    \end{tabular}
\end{table*}

\section{Continual Learning}
\label{sec:continual_learning}


In comparison to SRe$^2$L, our study implements a five-step class-incremental learning approach using ResNet-18 across CIFAR-10, CIFAR-100, and Tiny-ImageNet datasets, each with an \IPC setting of 10.
The results of these experiments are depicted in Figures \ref{fig:continual_learning_cifar10_ipc10}, \ref{fig:continual_learning_cifar100_ipc10}, and \ref{fig:continual_learning_tiny_ipc10} for CIFAR-10, CIFAR-100, and Tiny-ImageNet, respectively.

\begin{figure*}[t]
    \centering
    \includegraphics[width=0.6\linewidth,height=\textheight,keepaspectratio]{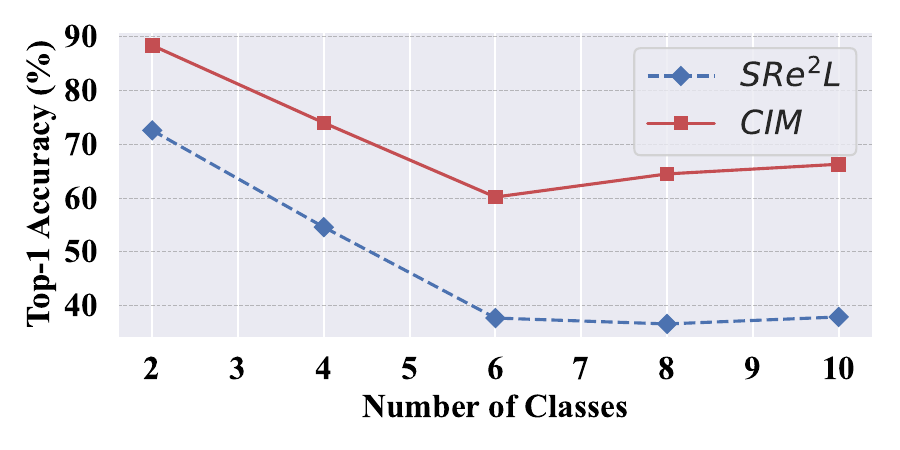}
    \caption{\small
        Visualization of continual learning on CIFAR-10 with $\IPC=10$.
    }
    \label{fig:continual_learning_cifar10_ipc10}
\end{figure*}

\begin{figure*}[t]
    \centering
    \includegraphics[width=0.6\linewidth,height=\textheight,keepaspectratio]{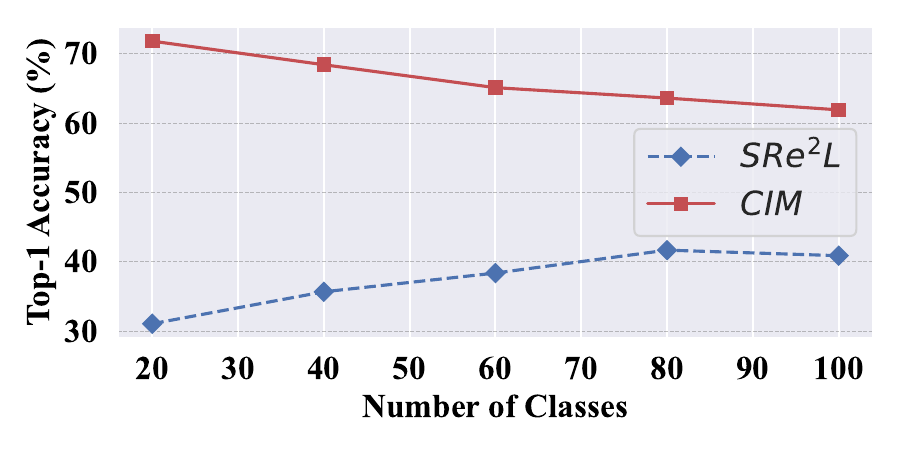}
    \caption{\small
        Visualization of continual learning on CIFAR-100 with $\IPC=10$.
    }
    \label{fig:continual_learning_cifar100_ipc10}
\end{figure*}

\begin{figure*}[t]
    \centering
    \includegraphics[width=0.6\linewidth,height=\textheight,keepaspectratio]{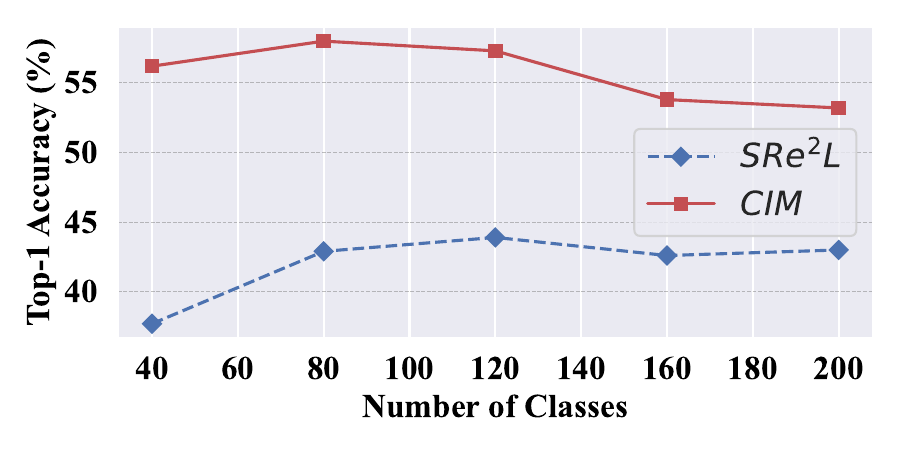}
    \caption{\small
        Visualization of continual learning on Tiny-ImageNet with $\IPC=10$.
    }
    \label{fig:continual_learning_tiny_ipc10}
\end{figure*}

\section{Visualization}
\label{sec:appendix_visualization}

\paragraph{Baselines.}

Within the scope of CIFAR-10 distillation under the $\IPC=10$ setting, we illustrate the visual representations of distilled datasets.
This includes visualizations for ADD~\cite{zhang2023accelerating} in Figure~\ref{fig:visualization_ADD}, DataDAM~\cite{sajedi2023datadam} in Figure~\ref{fig:visualization_DataDAM}, SRe$^2$L~\cite{yin2023squeeze} in Figure~\ref{fig:visualization_SRe2L}, and DREAM~\cite{liu2023dream} in Figure~\ref{fig:visualization_DREAM}. Distilled images of each method are generated starting from actual images, showcased in Figures~\ref{fig:visualization_datadam_init_n1} and \ref{fig:visualization_lic_init_n4}.

\paragraph{A simple squeezing-based method.}

The image-squeezing process entails resizing and concatenating images to facilitate dataset distillation. For example, consider the manipulation of $4$ images, each originally sized at $224 \times 224$ pixels. The initial step involves downsizing each image to $112 \times 112$ pixels. Subsequently, these reduced images are merged into a single composite image, effectively reverting to the original resolution of $224 \times 224$ pixels. This approach underpins a simplistic, squeezing-based dataset distillation method, whereby $N$ randomly selected original images are compressed into one distilled image to compose a condensed dataset. In the context of distilling CIFAR-10 with an $\IPC=10$ configuration, we exhibit the visual outcomes of this process for diverse settings: $N=1$ in Figure~\ref{fig:visualization_squeezed_n1}, $N=4$ in Figure~\ref{fig:visualization_squeezed_n4}, $N=9$ in Figure~\ref{fig:visualization_squeezed_n9}, $N=16$ in Figure~\ref{fig:visualization_squeezed_n16}, and $N=25$ in Figure~\ref{fig:visualization_squeezed_n25}.


\paragraph{Our proposed \algopt.}

With an \IPC setting of 10, we illustrate the distilled datasets generated by our proposed \algopt. These include visualizations for CIFAR-10 in Figure~\ref{fig:visualization_lic_cifar10}, CIFAR-100 in Figure~\ref{fig:visualization_lic_cifar100}, Tiny-ImageNet in Figure~\ref{fig:visualization_lic_tinyimagenet}, and ImageNet-1k in Figure~\ref{fig:visualization_lic_imagenet1k}.


\begin{figure}[p]
    \centering
    \includegraphics[width=1\linewidth,height=\textheight,keepaspectratio]{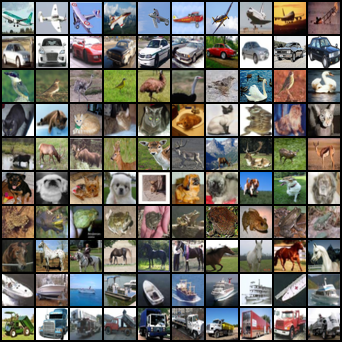}
    \caption{\small
        Visualization of initialized images before distilling on CIFAR-10.
    }
    \label{fig:visualization_datadam_init_n1}
\end{figure}

\begin{figure}[p]
    \centering
    \includegraphics[width=1\linewidth,height=\textheight,keepaspectratio]{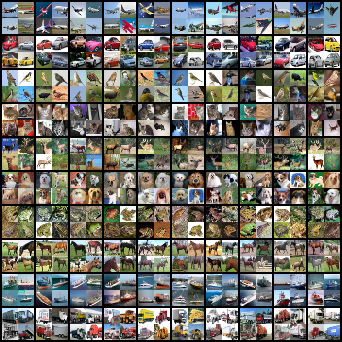}
    \caption{\small
        Visualization of initialized data before distilling on CIFAR-10 data showcases the mixture of $4$ images per initial instance.
    }
    \label{fig:visualization_lic_init_n4}
\end{figure}

\begin{figure}[p]
    \centering
    \includegraphics[width=1\linewidth,height=\textheight,keepaspectratio]{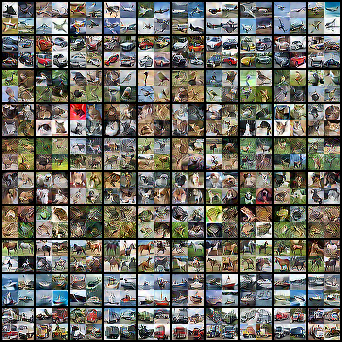}
    \caption{\small
        Synthetic data visualization on CIFAR-10 from ADD~\cite{zhang2023accelerating}.
    }
    \label{fig:visualization_ADD}
\end{figure}

\begin{figure}[p]
    \centering
    \includegraphics[width=1\linewidth,height=\textheight,keepaspectratio]{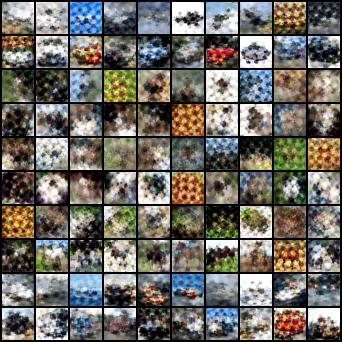}
    \caption{\small
        Synthetic data visualization on CIFAR-10 from DataDAM~\cite{sajedi2023datadam}.
    }
    \label{fig:visualization_DataDAM}
\end{figure}

\begin{figure}[p]
    \centering
    \includegraphics[width=1\linewidth,height=\textheight,keepaspectratio]{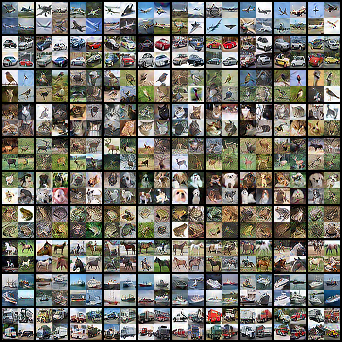}
    \caption{\small
        Synthetic data visualization on CIFAR-10 from DREAM~\cite{liu2023dream}.
    }
    \label{fig:visualization_DREAM}
\end{figure}

\begin{figure}[p]
    \centering
    \includegraphics[width=1\linewidth,height=\textheight,keepaspectratio]{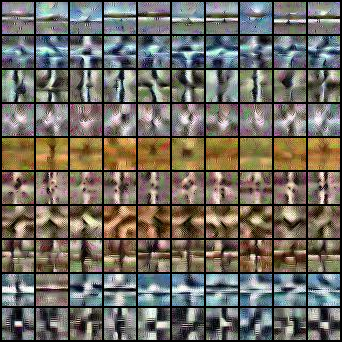}
    \caption{\small
        Synthetic data visualization on CIFAR-10 from SRe$^2$L \cite{yin2023squeeze}.
    }
    \label{fig:visualization_SRe2L}
\end{figure}

\begin{figure}[p]
    \centering
    \includegraphics[width=1\linewidth,height=\textheight,keepaspectratio]{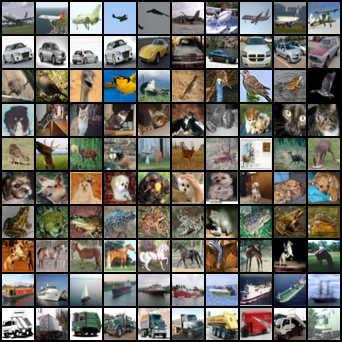}
    \caption{\small
        Initialized real data visualization on CIFAR-10 with $N=1$.
    }
    \label{fig:visualization_squeezed_n1}
\end{figure}

\begin{figure}[p]
    \centering
    \includegraphics[width=1\linewidth,height=\textheight,keepaspectratio]{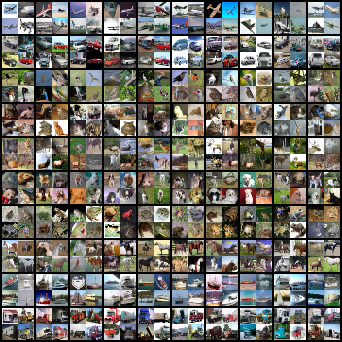}
    \caption{\small
        Initialized real data visualization on CIFAR-10 with $N=4$.
    }
    \label{fig:visualization_squeezed_n4}
\end{figure}

\begin{figure}[p]
    \centering
    \includegraphics[width=1\linewidth,height=\textheight,keepaspectratio]{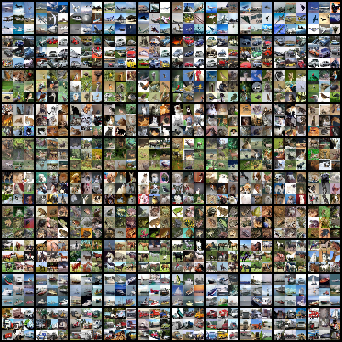}
    \caption{\small
        Initialized real data visualization on CIFAR-10 with $N=9$.
    }
    \label{fig:visualization_squeezed_n9}
\end{figure}

\begin{figure}[p]
    \centering
    \includegraphics[width=1\linewidth,height=\textheight,keepaspectratio]{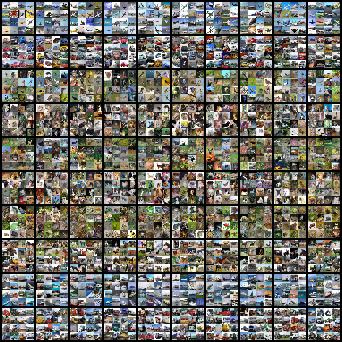}
    \caption{\small
        Initialized real data visualization on CIFAR-10 with $N=16$.
    }
    \label{fig:visualization_squeezed_n16}
\end{figure}

\begin{figure}[p]
    \centering
    \includegraphics[width=1\linewidth,height=\textheight,keepaspectratio]{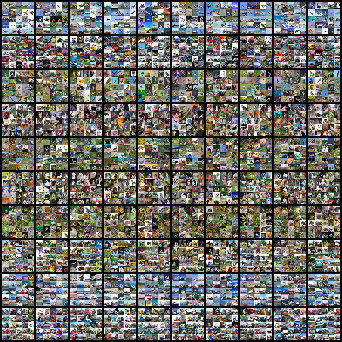}
    \caption{\small
        Initialized real data visualization on CIFAR-10 with $N=25$.
    }
    \label{fig:visualization_squeezed_n25}
\end{figure}

\begin{figure}[p]
    \centering
    \includegraphics[width=1\linewidth,height=\textheight,keepaspectratio]{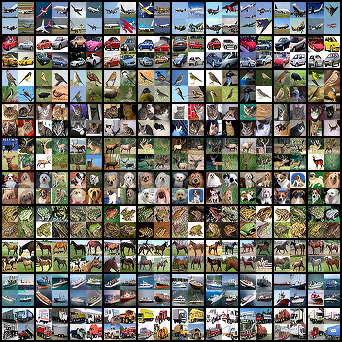}
    \caption{\small
        Synthetic data visualization on CIFAR-10 from \textbf{\algopt(Ours)}.
    }
    \label{fig:visualization_lic_cifar10}
\end{figure}


\begin{figure*}[p]
    \centering
    \begin{minipage}[t]{0.48\textwidth} 
        \centering
        \includegraphics[width=\linewidth,keepaspectratio]{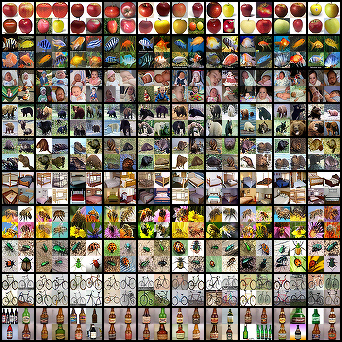}
        \caption{\small Synthetic data visualization on CIFAR-100 from \textbf{\algopt(Ours)}.}
        \label{fig:visualization_lic_cifar100}
        
        \vspace{1em} 
        
        \includegraphics[width=\linewidth,keepaspectratio]{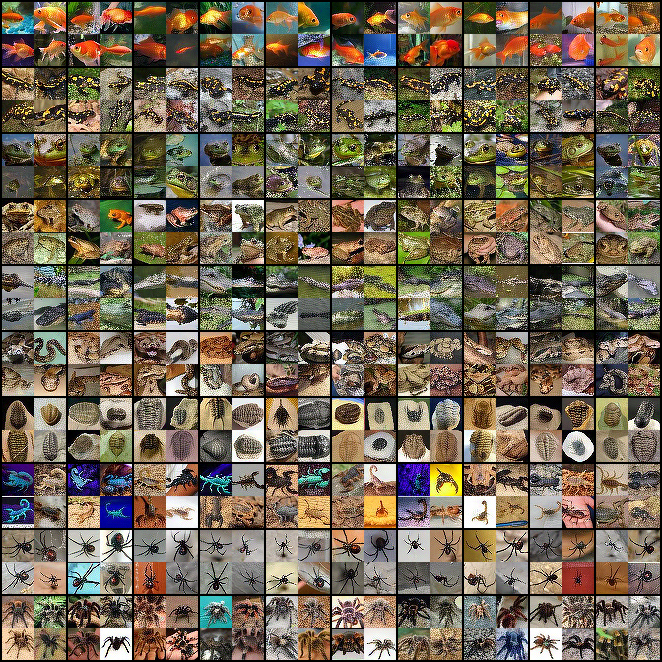}
        \caption{\small Synthetic data visualization on Tiny-ImageNet from \textbf{\algopt(Ours)}.}
        \label{fig:visualization_lic_tinyimagenet}
    \end{minipage}
    \hfill
    \begin{minipage}[t]{0.48\textwidth} 
        \centering
        \includegraphics[width=\linewidth,keepaspectratio]{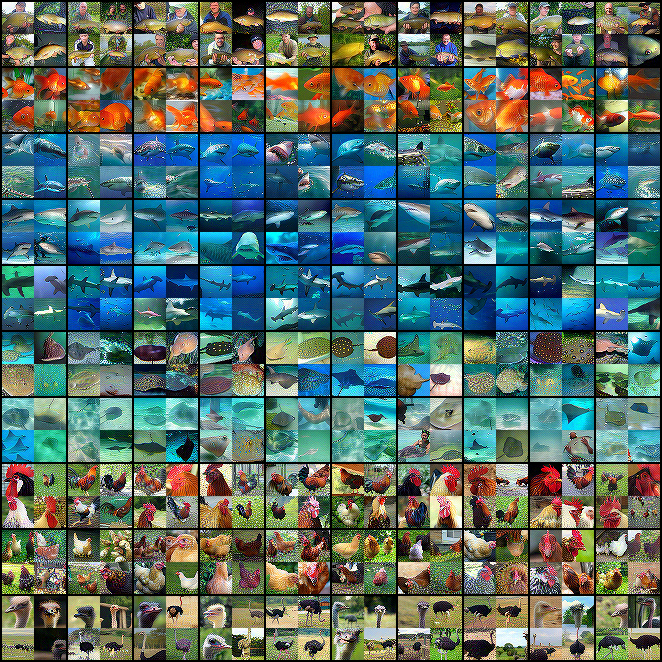}
        \caption{\small Synthetic data visualization on ImageNet-1k from \textbf{\algopt(Ours)}.}
        \label{fig:visualization_lic_imagenet1k}
    \end{minipage}
\end{figure*}


\end{document}